\documentclass{article}

\usepackage[final]{neurips_2023}

\usepackage[utf8]{inputenc} 
\usepackage[T1]{fontenc}    
\usepackage{hyperref}       
\usepackage{url}            
\usepackage{booktabs}       
\usepackage{amsfonts}       
\usepackage{nicefrac}       
\usepackage{microtype}      
\usepackage{xcolor}         

\usepackage{mathtools}
\usepackage{nicefrac}
\usepackage{amsmath}
\usepackage{amsthm,amssymb}
\usepackage{cleveref}
\usepackage{hyperref}
\hypersetup{
  colorlinks   = true, 
  urlcolor     = blue, 
  linkcolor    = blue, 
  citecolor   = blue 
}

\usepackage{amsmath}\allowdisplaybreaks
\usepackage{amsfonts,bm}
\usepackage{amssymb}
\usepackage{dsfont}
\usepackage{amsthm}
\usepackage{algorithm}
\usepackage{algorithmic}

\theoremstyle{plain}

\newtheorem{theorem}{Theorem}[section]

\newtheorem{lemma}{Lemma}[section]
\newtheorem{definition}{Definition}[section] 

\newtheorem{assumption}{Assumption}[section]
\newtheorem{remark}{Remark}[section]

\def\cA{{\mathcal{A}}}

\def\cD{{\mathcal{D}}}
\def\cE{{\mathcal{E}}}

\def\cM{{\mathcal{M}}}

\def\cO{{\mathcal{O}}}

\def\cR{{\mathcal{R}}}
\def\cS{{\mathcal{S}}}

\def\BE{{\mathbb{E}}}

\def\BI{{\mathbb{I}}}

\def\BP{{\mathbb{P}}}

\def\BR{{\mathbb{R}}}

\DeclareMathOperator*{\argmax}{arg\,max}
\DeclareMathOperator*{\argmin}{arg\,min}

\usepackage{color}
\usepackage{xcolor}

\newcommand{\mixCoe}{$\xi$}
\newcommand{\mixCoeff}{\xi}

\usepackage{tabularx}

\usepackage{pifont}       
\usepackage{bbding}       
\usepackage{fontawesome}  
\usepackage{array}
\newcolumntype{P}[1]{>{\centering\arraybackslash}p{#1}}

\usepackage{pgf}
\usepackage{tikz}
\usetikzlibrary{arrows,automata}
\usetikzlibrary{positioning}

\title{Learning Adversarial Low-rank Markov Decision Processes with Unknown Transition and Full-information Feedback}

\author{%
  Canzhe Zhao \\
   Shanghai Jiao Tong University\\
  \texttt{canzhezhao@sjtu.edu.cn} \\
  \And
  Ruofeng Yang \\
   Shanghai Jiao Tong University\\
  \texttt{wanshuiyin@sjtu.edu.cn} \\
  \And
  Baoxiang Wang \\
  The Chinese University of Hong Kong, Shenzhen\\
  \texttt{bxiangwang@cuhk.edu.cn} \\
  \And
  Xuezhou Zhang\\
  Boston University \\
  \texttt{xuezhouz@bu.edu} \\
  \AND
  Shuai Li\thanks{Corresponding author.} \\
Shanghai Jiao Tong University\\
  \texttt{shuaili8@sjtu.edu.cn} \\
}

\begin{document}

\maketitle

\begin{abstract}
In this work, we study the low-rank MDPs with adversarially changed losses in the full-information feedback setting. In particular, the unknown transition probability kernel admits a low-rank matrix decomposition \citep{REPUCB22}, and the loss functions may change adversarially but are revealed to the learner at the end of each episode. We propose a policy optimization-based algorithm POLO, and we prove that it attains the 
$\widetilde{O}(K^{\nicefrac{5}{6}}A^{\nicefrac{1}{2}}d\ln(1+M)/(1-\gamma)^2)$
regret guarantee, where $d$ is rank of the transition kernel (and hence the dimension of the unknown representations), $A$ is the cardinality of the action space, $M$ is the cardinality of the model class, and $\gamma$ is the discounted factor. Notably, our algorithm is oracle-efficient and has a regret guarantee with no dependence on the size of potentially arbitrarily large state space. Furthermore, we also prove an $\Omega(\frac{\gamma^2}{1-\gamma} \sqrt{d A K})$ regret lower bound for this problem, showing that low-rank MDPs are statistically more difficult to learn than linear MDPs in the regret minimization setting. To the best of our knowledge, we present the first algorithm that interleaves representation learning, exploration, and exploitation to achieve the sublinear regret guarantee for RL with nonlinear function approximation and adversarial losses.
\end{abstract}

\section{Introduction}\label{sec: intro}
In reinforcement learning (RL), the goal is to learn a (near) optimal policy through the interactions
between the learner and the environment, which is typically modeled as the Markov decision processes (MDPs) \citep{Feinberg96}.
When the state and action spaces are finite, several works have established the minimax (near) optimal regret guarantees for MDPs with finite horizon \citep{AzarOM17} and MDPs with infinite horizon \citep{Aristide2019near,HeZG21a}.  In real applications of RL, however, the state and action spaces may be arbitrarily large and even infinite, which may lead to the curse of dimensionality. To tackle this issue, a common approach is
\textit{function approximation}, which approximates the value functions of given policies with the leverage of feature mappings. Assuming that the feature mapping which embeds the state-action pairs to a low dimensional embedding space is known, RL with linear function approximation has been well-studied recently. 
In particular, linear mixture MDPs \citep{AyoubJSWY20} and linear MDPs \citep{JinYWJ20} are the two models of RL with linear function approximation that have been extensively studied. Notably, their (near) optimal regret guarantees are established by \citet{ZhouGS21} and \citet{He2023nearly} respectively. 
Nevertheless, in scenarios with complex and large-scale data, attaining the true underlying feature mappings might not be realistic, and thus representation learning is needed. Empirically, several works have shown that representation learning can accelerate the sample and computation efficiency of RL \citep{silver2018general,LaskinSA20,YangN21,StookeLAL21,SchwarzerRNACHB21,Xie2022pretraining}.
On the theoretical side, however, 
in sequential decision-making problems including RL, representation learning is more difficult than in its non-sequential and non-interactive counterpart (\textit{e.g.},  supervised learning) \citep{DuKWY20,WangFK21,WeiszAS21,REPUCB22}.
To permit sample-efficient RL in the presence of representation learning,
recent works have made initial attempts to study the theoretical
guarantees of representation learning in RL under the fixed or stochastic loss functions \citep{REPUCB22,ZhangSUWAS22}.

In practice, however, it might be stringent to assume that the loss functions are fixed or stochastic. To tackle this issue, \citet{Even-DarKM09,YuMS09}  propose the first algorithms with provably theoretical guarantees that can handle adversarial MDPs, where the loss functions may change adversarially in each episode.
Subsequently, most of the works in this line of research focus on learning tabular MDPs with adversarial loss functions
\citep{NeuGS10,NeuGSA10,NeuGS12,AroraDT12, ZiminN13,DekelH13,DickGS14,0002M19,RosenbergM19,JinL20,JinJLSY20,ShaniE0M20,ChenLW21,GhasemiHVT21,0002M21,JinHL21,Dai2022follow,ChenLR22}. To learn adversarial MDPs with large state and action spaces, recent works have also studied RL with adversarial loss functions and linear function approximation \citep{CaiYJW20,NeuO21,LuoWL21,Luo2021policy,HeZG22,zhao2023learning}. 
However, all these existing works assume that the feature mapping which embeds the state-action pairs into a low-dimensional space is known. 
As aforementioned, in complex and high-dimensional environments, the application of these algorithms may be still hindered due to the potential difficulty of knowing the true feature mappings a priori.
Therefore, the following question naturally remains open:

\textit{Can we devise an algorithm to simultaneously tackle the representation learning and adversarially changed loss functions in RL?}

In this work, we give an affirmative answer to the above question in the setting of adversarial low-rank MDPs with full-information feedback. Specifically, in this problem, the unknown transition probability kernel admits a low-rank matrix decomposition but the true representations regarding the transitions are not known a priori. Meanwhile, the loss functions are arbitrarily chosen by an adversary in different episodes and the loss function chosen for one episode is revealed to the learner at the end of the episode. 

To solve this problem, we propose a policy optimization-based algorithm, which
we call \textbf{P}olicy \textbf{O}ptimization for \textbf{LO}w-rank MDPs (POLO).
Specifically, POLO obtains an 
$\widetilde{O}(K^{\nicefrac{5}{6}}A^{\nicefrac{1}{2}}d\ln(1+M)/(1-\gamma)^2)$
regret guarantee for adversarial low-rank MDPs in the full-information feedback setting and is oracle-efficient.
In general, our algorithm POLO follows similar ideas of optimistic policy optimization methods in that it first constructs optimistic value function estimates and then runs online mirror descent (OMD) over the optimistic value estimates to deal with the adversarially changed loss functions \citep{ShaniE0M20,CaiYJW20,HeZG22,ChenLR22}. 
However, in the presence of representation learning, 
the exploration and exploitation needed to 
learn the adversarial MDPs
are more difficult than them in the tabular case \citep{ShaniE0M20,ChenLR22} and in the linear case \citep{CaiYJW20,HeZG22}.
In detail, to learn the underlying representation of the transition kernel, our algorithm performs maximum likelihood estimation (MLE) over the experienced transitions, as previous works learning low-rank MDPs \citep{Flambe20,REPUCB22,ZhangSUWAS22}.
Though the balance of representation learning, exploration, and exploitation can be simultaneously handled by previous algorithms for stochastic low-rank MDPs \citep{REPUCB22,ZhangSUWAS22}, these algorithms intrinsically have no regret guarantees but only sample complexity guarantees even in the setting of stochastic loss functions, since these algorithms need to take actions uniformly at certain steps in each episode (\textit{cf.}, Lemma 9 of \citet{REPUCB22}).\footnote{With the leverage the common explore-then-commit (ETC) style conversion, the modified versions of these algorithms can obtain sublinear regret in the setting of low-rank MDPs with stochastic loss functions, but this conversion is still not able to deal with adversarial loss functions.
} Hence, a straightforward adaption of their methods from stochastic setting to adversarial setting 
will also 
fail to learn
adversarial low-rank MDPs. 
To cope with this issue, we carefully devise an algorithm with a 
\textit{doubled exploration and exploitation} scheme, which interleaves (a) the exploration over transitions required in representation learning; and (b) the exploration and exploitation suggested by the policy optimization.
To this end, our algorithm adopts a mixed roll-out policy, which consists of a uniformly explorative policy and a policy optimized by OMD.
Through carefully tuning the hyper-parameter of the mixing coefficient used in our mixed policy, 
we can avoid pulling actions uniformly at random to conduct exploration in each episode and only conduct uniform exploration at a certain fraction of all the episodes (see \Cref{sec:hee} for details).
Besides, unlike tabular and linear (mixture) MDPs, it is in general hard to achieve the point-wise optimism for each state-action pair. Therefore, depart from previous methods \citep{ShaniE0M20,CaiYJW20,HeZG22} conducting policy optimization in the \textit{true} model,
our algorithm conducts policy optimization in the fixed \textit{learned} model with the epoch-based model update, which enables a new analysis scheme that only requires a \textit{near optimism} at the initial state $s_0$ (see \Cref{sec:omd} for details).
Also, we prove a regret lower bound of order
$\Omega(\frac{\gamma^2}{1-\gamma} \sqrt{d A K})$ for low-rank MDPs with fixed loss functions, which thus also serves as a regret lower bound for our problem and indicates that low-rank MDPs are statistically more difficult to learn than linear MDPs in the regret minimization setting.
To the best of our knowledge, this work makes the first step to establish an algorithm with a sublinear regret guarantee for adversarial low-rank MDPs, which permits  RL with both nonlinear function approximation and adversarial loss functions.
The concrete comparisons between the results of this work and those of previous works are summarized in \Cref{Table:literature_results}.

\vspace{-0.2cm}
\subsection{Additional Related Works}
\begin{table}[t]\caption{Comparisons of regret bounds with most related works studying adversarial RL with function approximation under unknown transitions. $K$ is the number of episodes, $d$ is the ambient dimension of the feature mapping, $\gamma$ is the discounted factor for infinite-horizon MDPs, and $S$, $A$, and $M$
are the cardinality of the state space, action space, and model class, respectively. Note that the dependence on $\gamma$ is not strictly comparable since some works originally studying finite-horizon MDPs and these results are translated into results for infinite-horizon MDPs by substituting horizon length $H$ with $\Theta(1/(1-\gamma))$. 
The column of ``unknown features'' indicates whether the algorithm can work in the case when no true feature mappings are known a priori.}
\vspace{-0.6cm}
    \label{Table:literature_results}
    \centering
\begin{tabularx}{\textwidth}{Xp{2.1cm}p{2.3cm}p{3.5cm}P{1.5cm}}
\label{table:literature results}\\
\toprule
Algorithm&Model  & \shortstack{Feedback} & Regret &Unknown Features\\
\midrule OPPO\newline\citep{CaiYJW20}&Linear Mixture MDPs&Full-information&$\widetilde{O}\left(d  \sqrt{K}/(1-\gamma)^2\right)$&\ding{55}\\
\midrule POWERS\newline\citep{HeZG22}&Linear Mixture MDPs&Full-information&$\widetilde{O}\left(d  \sqrt{K}/(1-\gamma)^{\nicefrac{3}{2}}\right)$&\ding{55}\\
\midrule
LSUOB-REPS \newline \cite{zhao2023learning} &Linear Mixture MDPs  &Bandit\newline Feedback & $\widetilde{O}\left(d S^2 \sqrt{K}+\sqrt{ \frac{S A K}{(1-\gamma)}}\right)$&\ding{55} \\
\midrule \cite{Luo2021policy} &Linear MDPs   &Bandit\newline Feedback & $\widetilde{O}\left(d^2 K^{\nicefrac{14}{15}}/(1-\gamma)^4\right)$&\ding{55} \\
\midrule \cite{0002LWZ23} &Linear MDPs   &Bandit\newline Feedback & $\widetilde{O}\left(\frac{A^{\nicefrac{1}{9}}d^{\nicefrac{2}{3}}K^{\nicefrac{8}{9}}}{(1-\gamma)^{\nicefrac{20}{9}}}\right)$&\ding{55} \\
\midrule PO-LSBE\newline \cite{ShermanKM23} &Linear MDPs   &Bandit\newline Feedback & $\widetilde{O}\left(\frac{dK^{\nicefrac{6}{7}}}{(1-\gamma)^2}+\frac{d^{\nicefrac{3}{2}}K^{\nicefrac{5}{7}}}{(1-\gamma)^4}\right)$&\ding{55} \\
\midrule
OPPO+ \newline \cite{zhong2023theoretical} &Linear MDPs  &Full-information & $\widetilde{O}\left(\frac{d^{\nicefrac{3}{4}}K^{\nicefrac{3}{4}}+d^{\nicefrac{5}{2}}\sqrt{K}}{(1-\gamma)^2} \right)$&\ding{55} \\
\midrule 
POLO \newline \textbf{(Ours)} &Low-rank MDPs  &Full-information & 
$\widetilde{O}\left(\frac{K^{\nicefrac{5}{6}}A^{\nicefrac{1}{2}}d\ln(1+M)}{(1-\gamma)^2}\right)$ \newline 
$\Omega\left(\frac{\gamma^2}{1-\gamma} \sqrt{d A K}\right)$ &\ding{52} \\
\bottomrule
\end{tabularx}
\vspace{-0.5cm}
\end{table}

\vspace{-0.2cm}
\paragraph{RL with Function Approximation}
Significant advances have emerged in RL with function approximation to cope with the curse of dimensionality in arbitrarily large state space or action space. In general, these results fall into two categories. The first category studies RL with linear function approximation, including linear MDPs \citep{YangW19,JinYWJ20,DuKWY20,ZanetteBBPL20,WangSY20,0001WDK21,HeZG21,HuCH22,He2023nearly} and linear mixture MDPs \citep{AyoubJSWY20,ZhangYJD21, ZhouGS21,HeZG21,Zhou2022horizon,WuZG22,MinH0G22,zhao2023learning}. Remarkably, \citet{He2023nearly} and \citet{ZhouGS21} obtain the nearly minimax optimal regret $\widetilde{O}(d \sqrt{H^3 K})$ in linear MDPs and linear mixture MDPs respectively when the loss functions are fixed or stochastic. The other category studies RL with general function approximation. Amongst these works, \citep{JiangKALS17,DannJKA0S18,SunJKA019,DuKJAD019,JinLM21}
study the MDPs satisfying the low Bellman-rank assumption, which assumes the Bellman error matrix has a low-rank factorization. Also, \citet{DuKLLMSW21} consider a similar but slightly more general assumption termed as bounded bilinear rank. 
Besides, \citet{RussoR13,WangSY20,JinLM21,IshfaqCNAYWPY21} study low Eluder dimension assumption, which is originally proposed to characterize the complexity of function classes for bandit problems. 

Representation learning in RL arises when the feature mapping that embeds the state-action pairs in RL with linear function approximation is no longer known a priori.
Such a problem is typically studied in the setting of low-rank MDPs, which does not assume the feature mapping of state-action pairs is known. Consequently, the setting of low-rank MDPs strictly generalizes the setting of linear MDPs, 
but at the cost of being more difficult to learn due to potential nonlinear function approximation induced by representation learning. 
In this line of research, algorithms with provably sample complexity guarantees have been developed in both model-based methods \citep{Flambe20,RenZSD22,REPUCB22} and model-free methods \citep{Modi2021modelfree,ZhangSUWAS22}, respectively. The model-based algorithms of \citet{Flambe20,RenZSD22,REPUCB22} learn the representation from a given model class of transition probability kernels. In contrast, the model-free methods do not require model learning but may bear some limitations.
In particular, \citet{Modi2021modelfree} assume the MDPs satisfying the minimal reachability assumption, and the sample complexity of the algorithm of \citet{ZhangSUWAS22} only holds for a special class of low-rank MDPs called block MDPs. 
Besides, representation learning in Markov games has also been investigated recently \citep{Ni2022representation}.


 \vspace{-0.7cm}
\paragraph{Rl with Adversarial Losses}
Recent years have witnessed significant advances in learning RL with adversarial losses in the tabular case \citep{NeuGS10,NeuGSA10,NeuGS12,AroraDT12, ZiminN13,DekelH13,DickGS14,0002M19,RosenbergM19,JinL20,JinJLSY20,ShaniE0M20,ChenLW21,GhasemiHVT21,0002M21,JinHL21,Dai2022follow,ChenLR22}.
When it comes to the setting of linear function approximation, 
various policy optimization-based methods have been established to solve adversarial linear mixture MDPs \citep{CaiYJW20,HeZG22} and adversarial linear MDPs \citep{Luo2021policy,LuoWL21,0002LWZ23,ShermanKM23,zhong2023theoretical}. 
Notably, \citet{HeZG22} establish the nearly minimax optimal regret bound $\widetilde{O}(d H^{\nicefrac{3}{2}} \sqrt{K})$ for adversarial linear mixture MDPs with full-information feedback. The insightful work of \citet{Luo2021policy} attains the first sublinear regret guarantee
$\widetilde{O}(d^2 H^4 K^{\nicefrac{14}{15}})$ in adversarial linear MDPs with bandit feedback, using policy optimization with \textit{dilated exploration bonuses}. Recently, the regret guarantee for the same setup has been improved to $\widetilde{O}\left(K^{\nicefrac{8}{9}}\right)$ and $\widetilde{O}\left(K^{\nicefrac{6}{7}}\right)$ by \citet{0002LWZ23} and \citet{ShermanKM23} (omitting all other dependences), respectively.
The other line of works studies RL with linear function approximation and adversarial losses using occupancy measure-based methods \citep{NeuO21,zhao2023learning}.
In specific, \citet{NeuO21} achieve the $\widetilde{O}(\sqrt{dHK})$ regret guarantee in adversarial linear MDPs with bandit feedback but known transition, and \cite{zhao2023learning} achieve the $\widetilde{O}(d S^2 \sqrt{K}+\sqrt{H S A K})$ regret for adversarial linear mixture MDPs with bandit feedback and unknown transition.
To the best of our knowledge, however, there are no works in existing literature studying RL with both nonlinear function approximation and adversarial loss functions.

\section{Preliminaries}\label{sec:setting}
We consider episodic infinite horizon low-rank MDPs with adversarial loss functions, the preliminaries of which are introduced as follows. 
\vspace{-0.2cm}
\paragraph{Episodic Infinite-horizon Adversarial MDPs} An episodic infinite horizon adversarial MDP is denoted by a tuple $(\mathcal{S}, \mathcal{A},P^\star,\{\ell_k\}_{k=1}^K,\gamma,d_0)$,\footnote{Though we focus on episodic infinite-horizon MDPs in this work, we note that it is not technically difficult to extend the analyses in this work to the case of episodic finite-horizon MDPs.} where $\mathcal{S}$ is the state space (with potentially infinitely many states), $\mathcal{A}$ is the finite action space with cardinality $|\mathcal{A}|=A$, $P^\star:\mathcal{S}\times \mathcal{A} \times \mathcal{S} \rightarrow [0,1]$ is the transition probability kernel such that $P^\star (s^\prime\mid s, a)$ is the probability of transferring to state $s^\prime$ from state $s$ after executing action $a$, $\gamma\in[0,1)$ is the discount factor, $d_0\in \Delta(\mathcal{S})$ is the initial distribution over state space, and $\ell_k: \mathcal{S} \times \mathcal{A} \rightarrow [0,1]$ is the loss function of episode $k$ chosen by the adversary. 
For the ease of exposition, we assume $d_0$ is known.

In this work, we consider a special class of MDPs called \textit{low-rank MDPs} \citep{Flambe20,REPUCB22,ZhangSUWAS22}. Specifically, instead of assuming the known true feature mapping, low-rank MDPs only assume that the transition probability kernel $P^\star$ admits a low-rank decomposition, with the formal definition given as follows.


\begin{definition}[Low-rank MDPs]\label{def:low rank mdp}
An MDP is a low-rank MDP if there exist two feature embedding functions $\phi^{\star}:\cS\times\cA\to \BR^d$, $\mu^{\star}:\cS\to \BR^d$ such that
for any $(s,a,s^{\prime}) \in \mathcal{S}\times \mathcal{A}\times\mathcal{S}$, $P^{\star}\left(s^{\prime} \mid s, a\right)=\mu^{\star}\left(s^{\prime}\right)^{\top} \phi^{\star}(s, a)$, where $\left\|\phi^\star(s, a)\right\|_2 \leq 1$ and for any function $g: \mathcal{S} \to[0,1],\left\|\int \mu^{\star}(s) g(s) \mathrm{d}(s)\right\|_2 \leq \sqrt{d}$.
\end{definition}

Note that the regularity assumption imposed over $\phi^{\star}$ and $\mu^{\star}$ is only for the purpose of normalization.


\paragraph{Function Approximation}
When the state space is arbitrarily large, function approximation is usually considered to permit sample-efficient learning for MDPs. Since the true feature mapping of state-action pairs is not known a priori in the low-rank MDPs, to make this problem tractable, we assume the access to a \textit{realizable} model class as previous works \citep{Flambe20,REPUCB22}, detailed in the following.
\begin{assumption}\label{ass:model_class}
There exists a known model class $\mathcal{M}=\{(\mu, \phi): \mu \in \Psi, \phi \in \Phi\}$ such that $\mu^{\star} \in \Psi$, $\phi^{\star} \in \Phi$, where for any $(s,a,s^{\prime}) \in \mathcal{S}\times \mathcal{A}\times\mathcal{S}$, $\mu\in\Psi$, $\phi\in \Phi$, $\left\|\phi(s, a)\right\|_2 \leq 1$, $\int \mu^{\top}\left(s^{\prime}\right) \phi(s, a) \mathrm{d}\left(s^{\prime}\right)=1$ and for any function $g: \mathcal{S} \to[0,1],\left\|\int \mu(s) g(s) \mathrm{d}(s)\right\|_2 \leq \sqrt{d}$.
\end{assumption}
Throughout this paper, for the sake of
brevity, we assume that the cardinality of
$\Psi$ and $\Phi$ are finite, meaning that $\cM$ also has bounded cardinality $M=|\cM|$. However, we note that extending the analyses to the function classes with infinite cardinality but bounded statistical complexity (\textit{e.g.}, classes with finite VC dimension) are not technically difficult.


\paragraph{Interaction Protocol}
We now introduce the interaction protocol between the learner and the environment.
To begin with, denote by $d_P^\pi(s, a)=(1-\gamma) \sum_{h=0}^{\infty} \gamma^h d_{P, h}^\pi(s, a)$ the state-action occupancy distribution, where $d_{P, h}^\pi(s, a)$ is the probability of visiting $(s,a)$ at step $h$ under some policy $\pi$ and transition $P$. With slight abuse of notation, let $d_P^\pi(s) = \sum_{a \in \mathcal{A}} d_P^\pi(s, a)$ be the state occupancy distribution, denoting the probability of visiting state $s$ under $\pi$ and $P$.

Ahead of time, an MDP is decided by the environment, and only the state space $\mathcal{S}$ and the action space $\mathcal{A}$ are revealed to the learner. Meanwhile, the adversary secretly chooses $K$ loss functions $\{\ell_k\}_{k=1}^K$, each of which will be used in one episode.
The interaction will proceed in $K$ episodes. 
At the beginning of episode $k$, the learner chooses a stochastic policy $\pi_k : \cS \times \cA\to [0, 1]$, where $\pi_k(a\mid s)$ is probability of taking $a$ at state $s$. 
Starting from an initial state $s_{0}\sim d_0$, the learner repeatedly
executes policy $\pi_k$ until reaching the termination.
After episode $k$ is terminated, the learner observes a trajectory $\{(s_{k,h},a_{k,h})\}_h$ as well as the loss function $\ell_k$.
To sample states from the state occupancy distribution $d^{\pi_k}_{P^\star}$, the learner can utilize a geometric sampling \textit{roll-in} procedure \citep{KakadeL02,AgarwalKLM21,REPUCB22}. In particular, for a given policy $\pi$, starting from an initial state $s_{0}\sim d_0$, at each step $h$, this roll-in procedure will terminate and return state $s_{h}$ with probability $1-\gamma$, and  otherwise will take action $a_h\sim\pi_k(\cdot\mid s_{h})$ and transfer to the next state $s_{h}\sim P^\star(\cdot\mid s_{h},a_{h})$.
It is then clear that the learner can sample $s\sim d^{\pi_k}_{P^\star}$ via invoking this sampling procedure.

For step $h$ in episode $k$ and for each state-action pair $(s,a) \in \mathcal{S} \times \mathcal{A}$, the state-action value $Q^\pi_{k}(s,a)$ and  the state value $V^\pi_{k}(s)$ under policy $\pi$ are defined as follows: $Q_{k}^{\pi}(s,a) = \mathbb{E}\left[\sum_{\tau=0}^{\infty} \gamma^{\tau}\ell_{k}(s_{k,\tau}, a_{k,\tau})\Big| \pi, P^\star, (s_{k,0}, a_{k,0}) = (s,a) \right]$ and $V_{k}^{\pi}(s) = \mathbb{E}\left[\sum_{\tau=0}^{\infty} \gamma^{\tau}\ell_{k}(s_{k,\tau}, a_{k,\tau})\Big| \pi, P^\star, s_{k,0}= s \right]$. Let $V_{k}^{\pi}=\BE_{s_0\sim d_0}[V_{k}^{\pi}(s_0)]$.
The learning objective is to minimize the \textit{expected regret} with respect to $\pi^\star$, defined as 
\begin{align*}
\cR_K = \mathbb{E}\left[\sum_{k=1}^K \left(V_k^{\pi_k}-V_k^{\pi^\star}\right)\right],
\end{align*}
where $\pi^\star \in \argmin_{\pi \in \Pi}\BE\left[\sum_{k=1}^K V_k^{\pi}\right]$ is the fixed optimal policy in hindsight and $\Pi$ is the set of all stochastic policies.

\section{Algorithm}\label{sec:algorithm}
\vspace{-0.2cm}
In this section, we present the proposed POLO algorithm, with the pseudocode illustrated in \Cref{algo:algo1}. 
At a high level, POLO
leverages a mixed roll-out policy to conduct doubled exploration and exploitation, \textit{i.e.}, (a) the exploration over transitions required by representation learning; and (b) the exploration and exploitation over adversarially changed loss functions required by policy optimization (\Cref{sec:hee}). 
To deal with the issue that only the \textit{near optimism} at the initial state $s_0$ is available in low-rank MDPs, POLO conducts policy optimization in fixed \textit{learned} models with the epoch-based model update, which
features a new analysis scheme (\Cref{sec:omd}).

\vspace{-0.2cm}
\subsection{Doubled Exploration and Exploitation}\label{sec:hee}
At the beginning of episode $k$, our algorithm first collects a state $s_k\sim d^{\tilde{\pi}_k}_{P^\star}$ by invoking the sampling procedure described in \Cref{sec:setting}. One of the key differences between our algorithm and previous works studying low-rank MDPs \citep{Flambe20,REPUCB22,ZhangSUWAS22} lies in how to interact with the environment after obtaining $s_k$. In specific, the core of the analyses in previous works relies on the \textit{one-step trick} that for any policy $\tilde{\pi}$ and any $g:\cS\times\cA\to[0,B]$,
\begin{align*}
    \mathbb{E}_{(s, a) \sim d_{P^{\star}}^{\tilde{\pi}}}[g(s, a)] \leq(1-\gamma)^{-1} \mathbb{E}_{(s, a) \sim d_{P^{\star}}^{\tilde{\pi}}}\left[\left\|\phi^{\star}(s, a)\right\|_{\Sigma_{\rho_k, \phi^{\star}}^{-1}}\right] \sqrt{k \gamma A \mathbb{E}_{\rho_k^{\prime}}\left[g^2(s, a)\right]+\gamma \lambda_k d B^2}\,,
\end{align*}
where $\rho_k(s, a)=1/k\sum_{i=1}^{k} d_{P^{\star}}^{\tilde{\pi}_i}(s, a)$ and $\rho_k^{\prime}(s,a)=1 / k \sum_{i=1}^{k} d^{\tilde{\pi}_i}_{P^\star}(s) U(a)$ with $U(\cdot)$ as the uniform distribution over $\cA$. 
This is critical to guarantee the (near) optimism of the estimated value functions and bound the estimation error of the unknown transition by the common elliptical potential lemma with respect to the true feature $\phi^{\star}(\cdot,\cdot)$. To enable the above one-step trick in the analyses, the algorithms in previous works conduct two-step exploration by sampling actions from $U(\cdot)$ in successive two steps after collecting $s_k\sim d^{\tilde{\pi}_k}_{P^\star}$. Consequently, though these algorithms enjoy excellent sample complexities, they intrinsically do not have regret guarantees due to the uniform exploration over action space, even in the stochastic setting.

Moreover, to deal with the adversarially changed loss functions, taking actions adaptively according to the observed loss functions in previous episodes, instead of uniformly taking actions, is required. To address this ``conflict'' so as to learn adversarial low-rank MDPs, we propose to use a mixed roll-out policy to interleave (a) the exploration over transitions required by representation learning; and (b) the exploration and exploitation over the adversarial loss functions by policy optimization, which we call \textit{doubled exploration and exploitation} and is pivotal to achieving our regret bound as we will shortly see. Formally, our algorithm will conduct the exploration over the transitions with probability $\mixCoeff$  and execute policy $\tilde{\pi}_k$ optimized by OMD with probability $1-\mixCoeff$, respectively (Line \ref{algo:line_bernoulli} - Line \ref{algo:line_mix2}). Subsequently, the newly collected data will be used to update the datasets (Line \ref{algo:line_collect_data}), and the empirical transition $\widehat{P}_k$ will be updated by performing MLE over the updated datasets by solving (Line \ref{algo:line_mle})
\begin{align}\label{eq:MLE}
    \left(\widehat{\mu}_k, \widehat{\phi}_k\right)=\underset{(\mu, \phi) \in \mathcal{M}}\argmax\, 
 \mathbb{E}_{\mathcal{D}_k\cup\mathcal{D}_k^{\prime}}\left[\ln \mu^{\top}\left(s^{\prime}\right) \phi(s, a)\right]\,,
\end{align}
where we denote $\mathbb{E}_{\mathcal{D}}\left[f\left(s, a, s^{\prime}\right)\right]=1 / |\cD| \sum_{\left(s, a, s^{\prime}\right) \in \mathcal{D}} f\left(s, a, s^{\prime}\right)$.

\begin{algorithm}[!thb]
\caption{Policy Optimization for Low-rank MDPs (POLO)}\label{algo:algo1}
\begin{algorithmic}[1]
\STATE \textbf{Input:} Mixing coefficient \mixCoe, epoch length $L$, regularization coefficients $\{\lambda_k\}_{k=1}^K$, bonus coefficients $\{\alpha_k\}_{k=1}^K$, model class $\mathcal{M}$, number of episodes $K$, learning rate $\eta$.
\STATE \textbf{Initialization:}
Set $\mathcal{D}_0=\emptyset$, $\mathcal{D}_0^{\prime}=\emptyset$.
\FOR{$i=1,2,\ldots,\lceil K/L\rceil$}
\STATE Set $k_i=(i-1)L+1$ and $\tilde{\pi}_{k_i}(\cdot \mid s)$ to be uniform for any $s\in\cS$.
\FOR{$k=k_i,k_i+1,\ldots,k_{i}+L-1$}
        \STATE Sample $s_k$ from $d_{P^\star}^{\tilde{\pi}_k}$.
        \STATE Sample $c_k\sim \operatorname{Ber}(1-\mixCoeff)$.\label{algo:line_bernoulli}
        \IF{$c_k=1$}\label{algo:line_mix1}
            \STATE Sample $a_k\sim \tilde{\pi}_k(\cdot\mid s_k), s_k^{\prime}\sim P^\star(\cdot\mid s_k, a_k), a_k^{\prime}\sim \tilde{\pi}_k(\cdot\mid s_k^{\prime}), s_k^{\prime\prime} \sim P^\star(\cdot\mid s_k^\prime, a_k^\prime)$.
        \ELSE
            \STATE Sample $a_k\sim U(\mathcal{A}), s_k^{\prime}\sim P^\star(\cdot\mid s_k, a_k), a_k^{\prime}\sim U(\mathcal{A}), s_k^{\prime\prime}  \sim P^\star(\cdot\mid s_k^\prime, a_k^{\prime})$.\label{algo:line_mix2}
        \ENDIF
    \STATE Observe the loss function $\ell_k$.
    \STATE Update datasets $\mathcal{D}_k=\mathcal{D}_{k-1}\cup\left\{\left(s_k, a_k, s^{\prime}_k\right)\right\}$, $\mathcal{D}_k^{\prime}=\mathcal{D}_{k-1}^{\prime}\cup\left\{\left(s^{\prime}_k, a^{\prime}_k, s_k^{\prime\prime}\right)\right\}$.\label{algo:line_collect_data}
    \IF{$k=k_i$}\label{algo:line:first_episode}
        \STATE Set the empirical transition $\widehat{P}_k(s^\prime\mid s,a)=\widehat{\mu}_k(s^\prime)^{\top}\widehat{\phi}_k(s,a)$, $\forall (s,a,s^\prime)\in\cS\times\cA\times\cS$, via solving Eq. \eqref{eq:MLE}.\label{algo:line_mle}
        \STATE Update the empirical covariance matrix $\widehat{\Sigma}_k=\sum_{(s, a)\in \mathcal{D}_k} \widehat{\phi}_k(s, a) \widehat{\phi}_k(s, a)^{\top}+\lambda_k I$.
        \STATE Set the bonus function $\widehat{b}_k(s, a)\coloneqq\operatorname{min} (\alpha_k \|\widehat{\phi}_k(s, a)\|_{\widehat{\Sigma}^{-1}_{k}}, 2)/(1-\gamma)$, $\forall (s,a)\in\cS\times\cA$.
    \ELSE
        \STATE Set the empirical transition $\widehat{P}_k=\widehat{P}_{k_i}$ and bonus function $\widehat{b}_k=\widehat{b}_{k_i}$.\label{algo:line:not_first_episode}
    \ENDIF
    \STATE Compute $\widehat{Q}_k^{\tilde{\pi}_k}(\cdot,\cdot) = \operatorname{Policy-Evaluation}(\widehat{P}_k,\ell_k-\widehat{b}_k,\tilde{\pi}_k)$.\label{algo:line_pe}
    \STATE Update policy $\tilde{\pi}_{k+1}(\cdot\mid\cdot)\propto \tilde{\pi}_{k}(\cdot\mid\cdot)\exp(-\eta\widehat{Q}_k^{\tilde{\pi}_k}(\cdot,\cdot))$.\label{algo:line_pi}
\ENDFOR
\ENDFOR
\end{algorithmic}
\end{algorithm}
\vspace{-0.2cm}
\subsection{Policy Optimization in Fixed Learned Models}\label{sec:omd}
It remains to compute the policy $\tilde{\pi}_{k+1}$ to be used in the next episode. To this end, we resort to the canonical OMD framework, which shares similar spirits with previous methods \citep{POMD20,CaiYJW20,HeZG22}.
However, previous OMD-based policy optimization methods for tabular and linear (mixture) MDPs \citep{POMD20,CaiYJW20,HeZG22} critically depend on the point-wise optimism for each state-action pair, \textit{i.e.}, $\widehat{Q}^{\tilde{\pi}_k}_k(s, a) \leq \ell_k(s, a)+\gamma[P^\star \widehat{V}^{\tilde{\pi}_k}_{k}](s, a)$, to enable the decomposition (\textit{cf.}, Lemma 1 by \citet{POMD20})
\begin{align*}
    \widehat{V}_k^{\tilde{\pi}_k}(s_0)-V_k^{\pi^\star}(s_0)
    &=\BE\left[ \sum_{\tau =0}^{\infty} \gamma^{\tau}\left\langle \tilde{\pi}_k(\cdot\mid s_{\tau})-\pi^\star(\cdot\mid s_{\tau}), \widehat{Q}_{k}^{\tilde{\pi}_k}(s_{\tau},\cdot)\right\rangle \,\middle|\, \pi^\star,P^{\star},s_0\right]\\
    &+\BE\left[ \sum_{\tau =0}^{\infty} \gamma^{\tau}\left(
    \widehat{Q}^{\tilde{\pi}_k}_k(s_{\tau}, a_{\tau}) - \ell_k(s_{\tau}, a_{\tau})-\gamma\left[P^\star \widehat{V}^{\tilde{\pi}_k}_{k}\right](s_{\tau}, a_{\tau})\right)\,\middle|\, \pi^\star,P^{\star},s_0\right]\,,
\end{align*}
where $\widehat{Q}^{\tilde{\pi}_k}_k$ is the state-action value function of $\tilde{\pi}_k$ on $(\widehat{P}_k,\ell_k-\widehat{b}_k)$ with $\widehat{b}_k$ as some bonus function and the expectation is taken over the randomness of sampling $a_{\tau} \sim \pi^{\star}\left(\cdot \mid s_{\tau}\right)$ and $s_{\tau+1} \sim P^{\star}\left(\cdot \mid s_{\tau}, a_{\tau}\right)$. The summation of the first term in the above display is contributed by competing with the optimal policy $\pi^\star$ in the \textit{true} model $P^\star$ and can be bounded by usual OMD analysis, which thus can be regarded as conducting policy optimization in the \textit{true} model. The point-wise optimism guarantees that the second term is less than or equal to $0$.  

Nevertheless, in low-rank MDPs, due to the unknown representation, it is generally hard to obtain the above point-wise optimism, which leaves the second optimism term unbounded. To cope with this issue, we instead consider the following decomposition:
\begin{align}\label{eq:algo_decomp1}
    &\widehat{V}_k^{\tilde{\pi}_k}(s_0)-V_k^{\pi^\star}(s_0)\notag\\
    =&\widehat{V}_k^{\tilde{\pi}_k}(s_0)-\widehat{V}_k^{\pi^\star}(s_0)+\widehat{V}_k^{\pi^\star}(s_0)-V_k^{\pi^\star}(s_0)\notag\\
    =&\BE\left[ \sum_{\tau =0}^{\infty} \gamma^{\tau}\left\langle \tilde{\pi}_k(\cdot\mid s_{k,\tau})-\pi^\star(\cdot\mid s_{k,\tau}), \widehat{Q}_{k}^{\tilde{\pi}_k}(s_{k,\tau},\cdot)\right\rangle \,\middle|\, \pi^\star,\widehat{P}_{k},s_0\right]
    +\widehat{V}_k^{\pi^\star}(s_0)-V_k^{\pi^\star}(s_0)\,,
\end{align}
where the first term is contributed by competing against the optimal policy $\pi^\star$ in the \textit{learned} model $\widehat{P}_k$ and can be seen as conducting policy optimization in \textit{learned} models.
This decomposition will be amenable as long as we can achieve a \textit{near optimism} at the initial state $s_0$, \textit{i.e.}, $\widehat{V}_k^{\pi^\star}(s_0)-V_k^{\pi^\star}(s_0)\lesssim 0$, which turns out to be feasible for low-rank MDPs \citep{REPUCB22}. However, there remains one more caveat. The first term in Eq. \eqref{eq:algo_decomp1} is now no longer directly bounded by OMD analysis, due to the local update nature of OMD-based policy optimization at each state and the state occupancy distribution $d^{\pi^{\star}}_{\widehat{P}_{k}}$ now varies across different episodes. To address this issue, \Cref{algo:algo1} adopts an epoch-based transition update, 
in which one epoch has $L$ episodes and the model is only updated at the first episode in one epoch (Line \ref{algo:line:first_episode} - Line \ref{algo:line:not_first_episode}).\footnote{Throughout this paper, we suppose for simplicity that the number of episodes $K$ is divisible by the epoch length $L$ considered.}
Concretely, \Cref{algo:algo1} sets $\widehat{P}_k=\widehat{P}_{k_i}$ and $\widehat{b}_k=\widehat{b}_{k_i}$, where $k_i$ is the first episode of the epoch to which the episode $k$ belongs.
In this manner, the learned model is \textit{fixed} in one epoch, and thus the regret of dealing with the adversarial loss functions by competing against the optimal policy $\pi^\star$ can be bounded in one epoch.
Subsequently, at the end of episode $k$, our algorithm first computes the optimistic value estimate $\widehat{Q}_k^{\tilde{\pi}_k}$ for current policy $\tilde{\pi}_k$ under $\widehat{P}_k$ together with the bonus-enhanced loss functions $\ell_k-\widehat{b}_k$ by policy evaluation (Line \ref{algo:line_pe}). Note that this boils down to planning in the setting of linear MDPs for given features in the learned model and
this
can be done computationally efficiently \citep{JinYWJ20}.
Then the policy is updated by solving 
\begin{align}\label{eq:OMD_update}
    \tilde{\pi}_{k+1}(\cdot\mid s) \in \argmin_{{\pi}(\cdot\mid s)\in\Delta(\cA)}\eta\left\langle{\pi}(\cdot\mid s), \widehat{Q}^{\tilde{\pi}_k}_k(s,\cdot)\right\rangle + D_F({\pi}(\cdot\mid s), \tilde{\pi}_k(\cdot\mid s))\,,
\end{align}
where $\eta>0$ is the learning rate to be tuned later and $D_F(x, y)=F(x)-F(y)-\langle x-y, \nabla F(y)\rangle$ is the Bregman divergence induced by the regularizer $F$. With $F({\pi}(\cdot\mid s))=\sum_{a\in\cA}{\pi}(a\mid s)\ln{\pi}(a\mid s)$ as the negative entropy, the closed-form solution to the above display is shown in Line \ref{algo:line_pi}, 
which can be regarded as a kind of soft policy improvement. 

\vspace{-0.2cm}
\section{Analysis}\label{sec: analysis}
\vspace{-0.2cm}
\subsection{Regret Upper Bound}
\vspace{-0.2cm}
The regret upper bound of our POLO algorithm for learning adversarial low-rank MDPs is guaranteed by the following theorem.
\begin{theorem}\label{thm: the upper bound}
For any adversarial low-rank MDP satisfying \Cref{def:low rank mdp}, by setting 
the epoch length $L=K^{\nicefrac{1}{2}}A^{\nicefrac{-1}{2}}d^{-1}\mixCoeff (1-\gamma)$,
learning rate $\eta= (1-\gamma)\sqrt{\ln A/(2L)}$, bonus coefficient $\alpha_k=O(\sqrt{\gamma(A/\mixCoeff+d^2)\ln(Mk/\delta)})$,  regularization coefficient $\lambda_k= O(d \ln (Mk/\delta))$, mixing coefficient $\mixCoeff = K^{\nicefrac{-1}{6}}A^{\nicefrac{1}{2}}d/(1-\gamma)$, and $\delta=1/K$, then the regret of \Cref{algo:algo1} is upper bounded by
\begin{align*}
\cR_K = O\left(\frac{K^{\frac{5}{6}}A^{\frac{1}{2}}d\ln\left(1+AMK^2\right)}{\left(1-\gamma\right)^2}\right)\,.
\end{align*}
\end{theorem}

\begin{remark}
Ignoring the dependence on all logarithmic factors but $M$, 
the regret upper bound can be simplified as
$\widetilde{O}(K^{\nicefrac{5}{6}}A^{\nicefrac{1}{2}}d\ln(1+M)/(1-\gamma)^2)$.
As we shall see in Section \ref{sec:lower_bound}, the regret upper bound in Theorem \ref{thm: the upper bound} matches the regret lower bound $\Omega(\frac{\gamma^2}{1-\gamma} \sqrt{d A K})$ in $A$ up to a logarithmic factor but looses in factors of $K$ and $d$.
Also, note that when $K$ is large enough such that $\mixCoeff$ and $L$ can be chosen as $\mixCoeff = K^{\nicefrac{-1}{6}}A^{\nicefrac{1}{3}}d^{\nicefrac{2}{3}}/(1-\gamma)$ and $L=K^{\nicefrac{1}{2}}A^{-1}d^{-2}\mixCoeff (1-\gamma)=K^{\nicefrac{1}{3}}A^{\nicefrac{-2}{3}}d^{\nicefrac{-4}{3}}\geq 1$, meaning that $K\geq d^4A^2$, the regret upper bound can be further optimized to $\widetilde{O}(K^{\nicefrac{5}{6}}A^{\nicefrac{1}{3}}d^{\nicefrac{2}{3}}\ln(1+M)/(1-\gamma)^2)$. However, this does not conflict with the regret lower bound in Section \ref{sec:lower_bound} since the magnitude of this upper bound is still larger than that of the regret lower bound as long as $K\geq A^{\nicefrac{1}{2}}d^{\nicefrac{-1}{2}}\gamma^6(1-\gamma)^3$.
\end{remark}

\subsection{Proof of Regret Upper Bound}\label{sec:upper_proof}
We now present the proof of \Cref{thm: the upper bound}.
To begin with, recall that in each episode $k$, after state $s_k$ is sampled from $d_{P^\star}^{\tilde{\pi}_k}$, the actual roll-out policy will be $\pi_k(\cdot\mid s)=\mixCoeff\cdot U(\cA)+(1-\mixCoeff)\cdot\tilde{\pi}_k(\cdot\mid s)$. Therefore, it holds that
\begin{align}\label{eq:reg_decompose1}
\cR_K &= \mathbb{E}\left[\sum_{k=1}^K \left(V_k^{{\pi}_k}-V_k^{\pi^\star}\right)\right]\notag\\
     &=  \mathbb{E}\left[\sum_{k=1}^K \BI\{c_k=1\}\left(V_k^{{\pi}_k}-V_k^{\pi^\star}\right)+\BI\{c_k=0\}\left(V_k^{{\pi}_k}-V_k^{\pi^\star}\right)\right]\notag\\
     &\leq \mathbb{E}\left[\sum_{k=1}^K \left(V_k^{\tilde{\pi}_k}-V_k^{\pi^\star}\right)\right]+\frac{\mixCoeff K}{(1-\gamma)}\,,
\end{align}
where the inequality is due to that $\sum_{h=0}^{\infty} \gamma^h \ell_k\left(s_h, a_h\right) \in[0,1/(1-\gamma)]$ holds for any episode $k$ and any trajectory $\{(s_h,a_h)\}_{h=0}^{\infty}$.
We now turn to bound the first term in Eq. \eqref{eq:reg_decompose1} by decomposing it into the following three terms
\begin{align}\label{eq:reg_decompose2}
\mathbb{E}\left[\sum_{k=1}^K \left(V_k^{\tilde{\pi}_k}-V_k^{\pi^\star}\right)\right]
     =\mathbb{E}\left[\underbrace{\sum_{k=1}^K \left(V_k^{\tilde{\pi}_k}-\widehat{V}_k^{\tilde{\pi}_k}\right)}_{\textsc{Estimation Bias Term}}
     +\underbrace{\sum_{k=1}^K \left(\widehat{V}_k^{\tilde{\pi}_k}-\widehat{V}_k^{\pi^\star}\right)}_{\textsc{OMD Regret Term}}
     +\underbrace{\sum_{k=1}^K \left(\widehat{V}_k^{\pi^\star}-V_k^{\pi^\star}\right)}_{\textsc{Optimism Term}}\right]\,.
\end{align}
\vspace{-0.2cm}
\paragraph{Bounding \textsc{OMD Regret Term}} 
The OMD regret term is contributed by competing against $\pi^\star$ using $\tilde{\pi}_k$ with $\widehat{Q}_{k}^{\tilde{\pi}_k}$ as loss function in the learned model $\widehat{P}_{k_i}$. 
This term is thus bounded by standard OMD analysis, detailed in the following lemma.
\begin{lemma}[OMD regret]\label{lem: the bound of OMD terms}
    By setting learning rate $\eta= (1-\gamma)\sqrt{\ln A/(2L)}$, the OMD regret term is bounded as $\BE\left[\sum_{k=1}^K \left(\widehat{V}_k^{\tilde{\pi}_k}-\widehat{V}_k^{\pi^\star}\right)\right]\leq \frac{K\sqrt{2\ln A}}{\sqrt{L}(1-\gamma)^2}$.
\end{lemma}
\vspace{-0.2cm}
\paragraph{Bounding \textsc{Optimism Term}} The optimism term is controlled by choosing appropriate bonus coefficient $\alpha_k$. Note that different from tabular and linear cases, the bonus functions and coefficients here are not devised to control the optimism for each state-action pair. Instead, they are devised to provide a (near) optimism only at the initial state $s_0$.
\begin{lemma}[Optimism] \label{lem: optimism}
By setting bonus coefficient $\alpha_k=O( \sqrt{\gamma(A/\mixCoeff+d^2)\ln(M k/\delta)})$, $ \lambda_k= O(d \ln (M k/ \delta))$, with probability $1-\delta$ , the optimism term is bounded as $\sum_{k=1}^K \left(\widehat{V}_k^{\pi^\star}-V_k^{\pi^\star}\right)\leq (L+\sqrt{K})\sqrt{\frac{A\ln(MN/\delta)}{\mixCoeff(1-\gamma)^3}}$.
\end{lemma}
\vspace{-0.2cm}
\paragraph{Bounding \textsc{Estimation Bias Term}} It remains to bound the estimation bias term, which comes from the difference between the values of running the same policy $\tilde{\pi}_k$ in the true model (\textit{i.e.}, $P^\star$ and $\ell_k$) and the learned empirical model (\textit{i.e.}, $\widehat{P}_k$ and $\ell_k-\widehat{b}_k$), respectively. This term can be translated into the error between the true model and the learned model using the common simulation lemma, which is thus bounded by the summation of bonus functions. Note that since the empirical features used to construct our bonus functions vary in each episode, we first relate the bonus functions with the fixed true feature $\phi^\star$ using the  one-step trick \citep{REPUCB22,ZhangSUWAS22}, and finally bound this term with the leverage of the canonical elliptical potential lemma. The result is shown in the following lemma.
\begin{lemma}[Estimation bias] \label{lem: estimation bias}
By setting bonus coefficient $\alpha_k=O( \sqrt{\gamma(A/\mixCoeff+d^2)\ln(M k/\delta)})$, $\lambda_k= O(d \ln (M k/ \delta))$, with probability $1-\delta$, the estimation bias term  is bounded as $\sum_{k=1}^K \left(V_k^{\tilde{\pi}_k}-\widehat{V}_k^{\tilde{\pi}_k}\right)\leq  O\left(\frac{d^2A\sqrt{KL}}{\mixCoeff(1-\gamma)^3}\sqrt{\ln(1+ K)\ln(M K/\delta)}\right)$.
\end{lemma}
\vspace{-0.2cm}
We refer the readers to Appendix \ref{sec:app_omitted_analysis} for the proof of the above lemmas. 
The proof of \Cref{thm: the upper bound} is now concluded by first combining Eq. \eqref{eq:reg_decompose1}, Eq. \eqref{eq:reg_decompose2}, Lemma \ref{lem: the bound of OMD terms}, \ref{lem: optimism}, and \ref{lem: estimation bias} and then choosing
$L=K^{\nicefrac{1}{2}}A^{\nicefrac{-1}{2}}d^{-1}\mixCoeff (1-\gamma)$,
$\mixCoeff = K^{\nicefrac{-1}{6}}A^{\nicefrac{1}{2}}d/(1-\gamma)$, and
$\delta=1/K$.

Intuitively, the epoch length $L$ illustrates a trade-off between dealing with the adversarial losses and the representation learning over the unknown transitions. When $L$ is large, there will be fewer restarts in the running of OMD and thus the learner will suffer less regret contributed by dealing with the adversarial losses as shown by Lemma \ref{lem: the bound of OMD terms}. In contrast, a smaller $L$ enables more frequent model updates, which leads to more accurate model estimation and less regret contributed by the representation learning as shown by Lemma \ref{lem: optimism} and \ref{lem: estimation bias}.
\vspace{-0.2cm}
\subsection{Regret Lower Bound}\label{sec:lower_bound}
This section presents the regret lower bound for learning adversarial low-rank MDPs with fixed loss functions in Theorem \ref{thm:lb}, which thus also serves as a regret lower bound for learning adversarial low-rank MDPs with full-information feedback. 
\begin{theorem}\label{thm:lb}
    Suppose $d\geq 8$, $S\geq d+1$, $A\geq d-3$, and $K\geq 2(d-4)A$. Then for any algorithm $\operatorname{Alg}$, there exists an episodic infinite-horizon low-rank MDP $\cM_{\operatorname{Alg}}$ with fixed loss function such that the expected regret for this MDP is lower bounded by $\Omega(\frac{\gamma^2}{1-\gamma} \sqrt{d A K})$.
\end{theorem}
\vspace{-0.4cm}
\begin{proof}[Proof Sketch]
    At a high level, we construct $dA$ hard-to-learn low-rank MDP instances, which are difficult to distinguish in KL divergence but have very different optimal policies. In particular, all the constructed low-rank MDP instances have three levels of states, in which the only state in the first level is a fixed initial state and the states in the third level are absorbing states. Moreover, only one unique absorbing state in the third level has the lowest loss, which is termed as the ``good state''.
    In the constructed low-rank MDP instance $\mathcal{M}_{\left(i^{\star}, a^{\star}\right)}$, the learner can only take specific action to transfer to state $s_{2, i^{\star}}$ in the second level and then take the other specific action to transfer to the unique good state. Due to the unknown representations of state-action pairs, the learner needs to distinguish all these $dA$ low-rank MDP instances, which is essentially equivalent to dealing with a bandit problem with $dA$ ``arms''. The detailed proof of Theorem \ref{thm:lb} is postponed to Appendix \ref{app:sec:lb}.
\end{proof}
\vspace{-0.2cm}
\begin{remark}
    \Cref{thm:lb}, to the best of our knowledge, provides the first regret lower bound for learning low-rank MDPs with fixed loss functions.
    We note that this regret lower bound can hold when $d\ll S$ and $d\ll A$, which thus means that this lower bound is non-trivial. Besides, the regret upper bound in our \Cref{thm: the upper bound} matches the regret lower bound in $A$ up to a logarithmic factor but looses a factor of $\widetilde{O}(K^{\nicefrac{1}{3}}d^{\nicefrac{1}{2}}/((1-\gamma)\gamma^2))$.
    Importantly, compared with the regret upper bound $\widetilde{O}(d \sqrt{ K/(1-\gamma)^3})$ of linear MDPs \citep{He2023nearly} (the finite horizon $H$ is substituted by the effective horizon $\Theta(1/(1-\gamma))$ in our infinite-horizon setting for a fair comparison), the dependence on $A$ in the regret lower bound of low-rank MDP shows a clear separation between low-rank MDPs and linear MDPs, which demonstrates that low-rank MDPs are statistically more difficult to learn than linear MDPs in the regret minimization setting.
    Also, we would like to note that similar hard MDP instances are first introduced to prove the regret lower bounds for tabular MDPs \citep{lattimore2020bandit,DominguesMKV21} and are recently also used to prove the lower bound of sample complexity for learning low-rank MDPs by \citet{ChengHL023}.
\end{remark}
\vspace{-0.2cm}
\section{Conclusions}\label{sec:conclusion}
\vspace{-0.2cm}
In this work, we study learning adversarial low-rank MDPs with unknown transition and full-information feedback. 
We prove that our proposed algorithm POLO achieves the 
$\widetilde{O}(K^{\nicefrac{5}{6}}A^{\nicefrac{1}{2}}d\ln(1+M)/(1-\gamma)^2)$
regret, which is the first sublinear regret guarantee for this challenging problem. 
The design of our proposed algorithm features (a) a doubled exploration and exploitation scheme to simultaneously learn the transitions and adversarial loss functions; and (b) policy optimization in the fixed learned models with epoch-based model update to enable a new analysis scheme 
that only requires the near optimism at the initial state instead of the point-wise optimism.
Also, we prove an $\Omega(\frac{\gamma^2}{1-\gamma} \sqrt{d A K})$ regret lower bound for this problem, serving as the first regret lower bound for learning low-rank MDPs in the regret minimization setting.
Besides, there also remain several interesting future directions to be explored. One natural question is whether it is possible to further optimize the dependence of our regret guarantee on the number of episodes $K$. The other question is how to also learn adversarial low-rank MDPs with only the bandit feedback available. This is also challenging since the current occupancy measure-based methods and policy optimization-based methods
tackling adversarial MDPs with bandit feedback 
both depend on the point-wise optimism provided by the true feature mapping,
which seems not feasible in low-rank MDPs. We hope our results may shed light on better understandings of RL with both nonlinear function approximation and adversarial losses and we leave the above extensions as our future works.


\section*{Limitations} We note that in general our algorithm is oracle-efficient (given access to the MLE computation oracle in Eq. \eqref{eq:MLE}) but may not be computationally efficient as previous works studying low-rank MDPs \citep{Flambe20,REPUCB22,ZhangSUWAS22,Ni2022representation}. However, we also remark that in practice, these algorithms including ours are computationally feasible since the computation of MLE is only a standard supervised learning problem and 
can be implemented using gradient descent methods. The other limitation is that throughout this paper, we assume model class $\cM$ with bounded cardinality $M$. This is standard in theoretical works studying RL with general function approximation \citep{JiangKALS17,SunJKA019}. Also, the regret upper bound of our algorithm only has a logarithmic dependence on $M$, which is also standard in the literature. Moreover, we remark that extending the analyses to an infinite hypothesis class is possible if the hypothesis class has bounded statistical complexity \citep{Flambe20}.


\bibliographystyle{plainnat}
\bibliography{neurips_2023}

\clearpage
\appendix
\section*{Appendix}
\section{Omitted Analysis of The Regret Upper Bound}\label{sec:app_omitted_analysis}
In this section, we first introduce some further notations and then present the detailed analysis of \Cref{thm: the upper bound}. For the proofs in this section,
we assume that $K$ is divisible by the epoch length $L$ considered for simplicity.

To begin with, denote by $\rho_k(s)=1/k
\sum_{i=1}^{k} d_{P^{\star}}^{\tilde{\pi}_i}(s)$ the averaged state occupancy distribution
under $\{\tilde{\pi}_i\}_{i=1}^{k}$ and $P^\star$. Analogously, let $\rho_k^{\prime}(s^\prime)=\sum_{(s,a)\in\cS\times\cA}\rho_k(s) \bar{\pi}_k(a\mid s)P^{\star}\left(s^{\prime} \mid s, a\right)$ be the average next-state occupancy distribution after $s$ is sampled from 
$d_{P^{\star}}^{\tilde{\pi}_i}$, where $\bar{\pi}_k(\cdot\mid s)=\mixCoeff\cdot U(\cA)+(1-\mixCoeff)\cdot 1/k\sum_{i=1}^{k}\tilde{\pi}_{i}(\cdot\mid s)$ is the averaged mixed roll-out policy. 
Define $f_k(s, a)=\left\|\widehat{P}_k(\cdot \mid s, a)-P^{\star}(\cdot \mid s, a)\right\|_1$ as the $\ell_1$-error between the estimate transition kernel $\widehat{P}_k$ and $P^\star$.
Further, we define the following feature covariance matrices:
\begin{itemize}
    \item $\widehat{\Sigma}_{k, \phi}=k \mathbb{E}_{(s, a) \sim \mathcal{D}_k}\left[\phi(s,a) \phi(s,a)^{\top}\right]+\lambda_k I$,
    \item $\Sigma_{\rho_k, \phi}=k \mathbb{E}_{(s, a) \sim \rho_k}\left[\phi(s, a) \phi(s, a)^{\top}\right]+\lambda_k I$,
    \item $\Sigma_{\rho_k \times \bar{\pi}_k, \phi}=k \mathbb{E}_{s \sim \rho_k, a \sim \bar{\pi}_k}\left[\phi(s, a) \phi(s, a)^{\top}\right]+\lambda_k I$.
\end{itemize}
For notational convenience, we abbreviate $\widehat{\Sigma}_{k, \phi_k}$ as $\widehat{\Sigma}^{-1}_{k}$. Also, note that $\widehat{\Sigma}_{k, \phi}$ is an unbiased estimate of $\Sigma_{\rho_k \times \bar{\pi}_k, \phi}$.

\subsection{Bounding \textsc{OMD Regret Term}}
We first present the proof of \Cref{lem: the bound of OMD terms}, which follows from the standard OMD analysis.
\begin{proof}[Proof of \Cref{lem: the bound of OMD terms}]
We first consider fixed initial state $s_0$.
For some fixed $i\in[N]$, one can see that the OMD regret in episode $k\in\{k_i,k_i+1,\ldots,k_i+L-1\}$ can be written recursively as
\begin{align*}
    &\widehat{V}_k^{\tilde{\pi}_k}(s_0)-\widehat{V}_k^{\pi^\star}(s_0)\\
    =&\left\langle \tilde{\pi}_k\left(\cdot\mid s_0\right), \widehat{Q}_k^{\tilde{\pi}_k}\left(s_0,\cdot\right)\right\rangle-\left\langle\pi^\star\left(\cdot\mid s_0\right), \widehat{Q}_k^{\pi^\star}\left(s_0,\cdot\right)\right\rangle \\
    =&\left\langle \tilde{\pi}_k\left(\cdot\mid s_0\right)-\pi^\star\left(\cdot\mid s_0\right), \widehat{Q}_k^{\tilde{\pi}_k}\left(s_0,\cdot\right)\right\rangle+\left\langle\pi^\star\left(\cdot\mid s_0\right),
\widehat{Q}_k^{\tilde{\pi}_k}\left(s_0,\cdot\right)-\widehat{Q}_k^{\pi^\star}\left(s_0,\cdot\right)\right\rangle \\
    =&\left\langle \tilde{\pi}_k(\cdot\mid s_0)-\pi^\star(\cdot\mid s_0), \widehat{Q}_k^{\tilde{\pi}_k}(s_0,\cdot)\right\rangle + \left\langle \pi^\star(\cdot\mid s_0) , \gamma\left[\widehat{P}_{k_i}\left(\widehat{V}^{\tilde{\pi}_k}_{k}-\widehat{V}^{\pi^\star}_{k}\right)\right](s_0,\cdot)\right\rangle\\
    =&\BE\left[ \sum_{\tau =0}^{\infty} \gamma^{\tau}\left\langle \tilde{\pi}_k(\cdot\mid s_{k,\tau})-\pi^\star(\cdot\mid s_{k,\tau}), \widehat{Q}_{k}^{\tilde{\pi}_k}(s_{k,\tau},\cdot)\right\rangle \,\middle|\, \pi^\star,\widehat{P}_{k_i},s_0\right]\,,
\end{align*}
where recall $\widehat{Q}_{k}^{\pi}(s,a) = \mathbb{E}\left[\sum_{\tau=0}^{\infty} \gamma^{t}[\ell_{k}-\widehat{b}_k](s_{k,\tau}, a_{k,\tau})\Big| \pi, \widehat{P}_{k_i}, (s_{k,0}, a_{k,0}) = (s,a) \right]$ is the state-action value of policy $\pi$ under the empirical model $(\widehat{P}_{k_i},\ell_{k}-\widehat{b}_{k_i})$ and the expectation $\BE[\cdot\mid  \pi^\star,\widehat{P}_{k_i},s_0]$ is taken over the randomness of the state-action sequence $\{(s_{k,\tau},a_{k,\tau})\}_{\tau=0}^{\infty}$ with $a_{k,\tau}\sim\pi^\star(\cdot\mid s_{k,\tau})$, $s_{k,\tau+1}\sim \widehat{P}_{k_i}(\cdot\mid s_{k,\tau},a_{k,\tau})$, and $s_{k,0}=s_0$.

Taking summation of the above display over $k\in\{k_i,k_i+1,\ldots,k_i+L-1\}$ and re-arranging show that
\begin{align}\label{eq:OMD_reg1}
    &\BE\left[\sum_{k=k_i}^{k_i+L-1} \left(\widehat{V}_k^{\tilde{\pi}_k}(s_0)-\widehat{V}_k^{\pi^\star}(s_0)\right)\right]\notag\\
    =&\BE\left[\BE\left[\sum_{k=k_i}^{k_i+L-1}\sum_{\tau =0}^{\infty}  \gamma^{\tau}\left\langle \tilde{\pi}_k(\cdot\mid s_{k,\tau})-\pi^\star(\cdot\mid s_{k,\tau}), \widehat{Q}_{k}^{\tilde{\pi}_k}(s_{k,\tau},\cdot)\right\rangle \,\middle|\, \pi^\star,\widehat{P}_{k_i},s_0\right]\right]\notag\\
    =&\sum_{\tau =0}^{\infty}\gamma^{\tau}\BE\left[\BE\left[\sum_{k=k_i}^{k_i+L-1}\left\langle \tilde{\pi}_k(\cdot\mid s_{k,\tau})-\pi^\star(\cdot\mid s_{k,\tau}), \widehat{Q}_{k}^{\tilde{\pi}_k}(s_{k,\tau},\cdot)\right\rangle \,\middle|\, \pi^\star,\widehat{P}_{k_i},s_0\right]\right]\,.
\end{align}

Further, note that the update process of policy in  Eq. (\ref{eq:OMD_update}) can be solved by the following two-step procedure \citep{lattimore2020bandit}:
\begin{align}
\widehat{\pi}_{k+1}(\cdot\mid s) &\in \argmin_{{\pi}(\cdot\mid s)\in\BR^{A}_{+}}\eta\left\langle{\pi}(\cdot\mid s), \widehat{Q}^{\tilde{\pi}_k}_k(s,\cdot)\right\rangle + D_F({\pi}(\cdot\mid s), \tilde{\pi}_k(\cdot\mid s))\label{eq:omd_solution1}\quad\text{and}\\
\tilde{\pi}_{k+1}(\cdot\mid s) &\in \argmin_{{\pi}(\cdot\mid s)\in\Delta(\cA)} D_F({\pi}(\cdot\mid s), \widehat{\pi}_{k+1}(\cdot\mid s))\label{eq:omd_solution2}\,,
\end{align}
for any $s\in\cS$. Eq. \eqref{eq:omd_solution1} combined with the first-order optimality condition implies that
\begin{align}
    \widehat{Q}^{\tilde{\pi}_k}_k(s,\cdot) = -\frac{1}{\eta}(\nabla F(\widehat{\pi}_{k+1}(\cdot\mid s))-\nabla F(\tilde{\pi}_k(\cdot\mid s)))\,, \label{eq:OMD_first}
\end{align}
for any $s\in\cS$.
This display shows that $\widehat{\pi}_{k+1}(a\mid s)=\tilde{\pi}_k(a\mid s)\exp (-\eta\widehat{Q}^{\tilde{\pi}_k}_k(s,\cdot))$, and  $\widehat{\pi}_{k+1}(a\mid s)\leq \tilde{\pi}_k(a\mid s)$ since $\widehat{Q}^{\tilde{\pi}_k}_k(s,a)\geq 0$, for any $(s,a)\in\cS\times\cA$.

Therefore, one can see that
\begin{align}\label{eq:OMD_eqx}
    &\BE\left[\left\langle \tilde{\pi}_k(\cdot\mid s_{k,\tau})-\pi^\star(\cdot\mid s_{k,\tau}), \widehat{Q}_{k}^{\tilde{\pi}_k}(s_{k,\tau},\cdot)\right\rangle\middle|\, \pi^\star,\widehat{P}_{k_i},s_0\right]\notag\\
    =&\BE\left[\frac{1}{\eta}\langle \pi^\star(\cdot\mid s_{k,\tau})-\tilde{\pi}_k(\cdot\mid s_{k,\tau}), \nabla F(\widehat{\pi}_{k+1}(\cdot\mid s_{k,\tau}))-\nabla F(\tilde{\pi}_k(\cdot\mid s_{k,\tau})) \rangle\middle|\, \pi^\star,\widehat{P}_{k_i},s_0\right]\notag\\
    =& \BE\left[\frac{1}{\eta}(D_F(\pi^\star(\cdot\mid s_{k,\tau}), \tilde{\pi}_k(\cdot\mid s_{k,\tau}))+D_F(\tilde{\pi}_k(\cdot\mid s_{k,\tau}), \widehat{\pi}_{k+1}(\cdot\mid s_{k,\tau}))\right.\notag\\
    &\left.-D_F(\pi^\star(\cdot\mid s_{k,\tau}),\widehat{\pi}_{k+1}(\cdot\mid s_{k,\tau})))\middle|\, \pi^\star,\widehat{P}_{k_i},s_0\right]\notag\\
    \leq& \BE\left[\frac{1}{\eta}\left(D_F(\pi^\star(\cdot\mid s_{k,\tau}), \tilde{\pi}_k(\cdot\mid s_{k,\tau}))+D_F(\tilde{\pi}_k(\cdot\mid s_{k,\tau}), \widehat{\pi}_{k+1}(\cdot\mid s_{k,\tau}))\right.\right.\notag\\
    &\left.\left.-D_F\left(\pi^\star(\cdot\mid s_{k,\tau}),\tilde{\pi}_{k+1}(\cdot\mid s_{k,\tau})\right)\right)\middle|\, \pi^\star,\widehat{P}_{k_i},s_0\right]\,,
\end{align}
where the first equality comes from Eq. (\ref{eq:OMD_first}), the second equality is due to the three-point lemma, and the last inequality follows the generalized Pythagorean theorem.

Taking summation of Eq. \eqref{eq:OMD_eqx} over $k\in\{k_i,k_i+1,\ldots,k_i+L-1\}$ leads to
\begin{align}\label{eq:eq:OMD_reg1_1}
    &\BE\left[\sum_{k=k_i}^{k_i+L-1}\left\langle \tilde{\pi}_k(\cdot\mid s_{k,\tau})-\pi^\star(\cdot\mid s_{k,\tau}), \widehat{Q}_{k}^{\tilde{\pi}_k}(s_{k,\tau},\cdot)\right\rangle\middle|\, \pi^\star,\widehat{P}_{k_i},s_0\right]\notag\\
    \leq&\BE\left[\frac{1}{\eta}\left(D_F(\pi^\star(\cdot\mid s_{k_i,\tau}), \tilde{\pi}_{k_i}(\cdot\mid s_{k_i,\tau}))
    +\sum_{k=k_i}^{k_i+L-1}D_F(\tilde{\pi}_k(\cdot\mid s_{k,\tau}), \widehat{\pi}_{k+1}(\cdot\mid s_{k,\tau}))\right)\middle|\, \pi^\star,\widehat{P}_{k_i},s_0\right]\,.
\end{align}

The first term in Eq. \eqref{eq:eq:OMD_reg1_1} can be bounded as follows: 
\begin{align}\label{eq:eq:OMD_reg2}
&\BE\left[D_F(\pi^\star(\cdot\mid s_{k_i,\tau}), \tilde{\pi}_{k_i}(\cdot\mid s_{k_i,\tau}))\middle|\, \pi^\star,\widehat{P}_{k_i},s_0\right]\notag\\
=& \BE\left[\sum_{a \in \cA} \pi^\star(a\mid s_{k_i,\tau}) \ln \frac{\pi^\star(a\mid s_{k_i,\tau})}{\tilde{\pi}_{k_i}(a\mid s_{k_i,\tau})}\middle|\, \pi^\star,\widehat{P}_{k_i},s_0\right]\notag\\
\leq& \BE\left[\sum_{a \in \cA} \pi^\star(a\mid s_{k_i,\tau})\ln A\middle|\, \pi^\star,\widehat{P}_{k_i},s_0\right]\notag\\
=&\ln A\,,
\end{align}
where the inequality is because we choose $\tilde{\pi}_{k_i}(\cdot\mid s)=U(\cA)$ for any $s\in\cS$ in \Cref{algo:algo1}.


It remains to bound $\BE\left[\sum_{k=k_i}^{k_i+L-1}D_F(\tilde{\pi}_k(\cdot\mid s_{k,\tau}), \widehat{\pi}_{k+1}(\cdot\mid s_{k,\tau}))\middle|\, \pi^\star,\widehat{P}_{k_i},s_0\right]$ in Eq. \eqref{eq:eq:OMD_reg1_1}:
\begin{align}\label{eq:eq:OMD_reg3}
    &\BE\left[\sum_{k=k_i}^{k_i+L-1}D_F(\tilde{\pi}_k(\cdot\mid s_{k,\tau}), \widehat{\pi}_{k+1}(\cdot\mid s_{k,\tau}))\middle|\, \pi^\star,\widehat{P}_{k_i},s_0\right]\notag\\
    =& \BE\left[\sum_{k=k_i}^{k_i+L-1}\left(-D_F(\widehat{\pi}_{k+1}(\cdot\mid s_{k,\tau}), \tilde{\pi}_{k}(\cdot\mid s_{k,\tau}))\notag\right.\right.\\
    &\left.\left.+\left\langle\nabla F(\tilde{\pi}_k(\cdot\mid s_{k,\tau}))-\nabla F(\widehat{\pi}_{k+1}(\cdot\mid s_{k,\tau})), \tilde{\pi}_k(\cdot\mid s_{k,\tau})-\widehat{\pi}_{k+1}(\cdot\mid s_{k,\tau})\right\rangle\right)\middle|\, \pi^\star,\widehat{P}_{k_i},s_0\right]\notag\\
    \leq& \BE\left[\sum_{k=k_i}^{k_i+L-1} \left(-D_F(\widehat{\pi}_{k+1}(\cdot\mid s_{k,\tau}), \tilde{\pi}_{k}(\cdot\mid s_{k,\tau}))+\frac{1}{2}\left\|\nabla F(\tilde{\pi}_k(\cdot\mid s_{k,\tau}))-\nabla F(\widehat{\pi}_{k+1}(\cdot\mid s_{k,\tau}))\right\|^2_{\nabla^{-2}F(z_k(\cdot\mid s_{k,\tau}))}\right.\right.\notag\\
    &\left.\left.+\frac{1}{2}\|\tilde{\pi}_k(\cdot\mid s_{k,\tau})-\widehat{\pi}_{k+1}(\cdot\mid s_{k,\tau})\|^2_{\nabla^{2}F(z_k(\cdot\mid s_{k,\tau}))}\right)\notag\middle|\, \pi^\star,\widehat{P}_{k_i},s_0\right]\\
    =& \BE\left[\sum_{k=k_i}^{k_i+L-1} \left(-D_F(\widehat{\pi}_{k+1}(\cdot\mid s_{k,\tau}), \tilde{\pi}_{k}(\cdot\mid s_{k,\tau}))+\frac{1}{2}\|\eta\widehat{Q}_{k}^{\tilde{\pi}_k}(s_{k,\tau},\cdot)\|^2_{\nabla^{-2}F(z_k(\cdot\mid s_{k,\tau}))}\right.\right.\notag\\
    &\left.\left.+\frac{1}{2}\|\tilde{\pi}_k(\cdot\mid s_{k,\tau})-\widehat{\pi}_{k+1}(\cdot\mid s_{k,\tau})\|^2_{\nabla^{2}F(z_k(\cdot\mid s_{k,\tau}))}\right)\notag\middle|\, \pi^\star,\widehat{P}_{k_i},s_0\right]\\
    =& \BE\left[\sum_{k=k_i}^{k_i+L-1} \left(-\frac{1}{2}\|\tilde{\pi}_k(\cdot\mid s_{k,\tau})-\widehat{\pi}_{k+1}(\cdot\mid s_{k,\tau})\|^2_{\nabla^{2}F(\omega_k(\cdot\mid s_{k,\tau}))}
    +\frac{1}{2}\|\eta\widehat{Q}_{k}^{\tilde{\pi}_k}(s_{k,\tau},\cdot)\|^2_{\nabla^{-2}F(z_k(\cdot\mid s_{k,\tau}))}\right.\right.\notag\\
    &\left.\left.+\frac{1}{2}\|\tilde{\pi}_k(\cdot\mid s_{k,\tau})-\widehat{\pi}_{k+1}(\cdot\mid s_{k,\tau})\|^2_{\nabla^{2}F(z_k(\cdot\mid s_{k,\tau}))}\right)\notag\middle|\, \pi^\star,\widehat{P}_{k_i},s_0\right]\\
    =& \BE\left[\frac{\eta^2}{2}\sum_{k=k_i}^{k_i+L-1}\sum_{a\in \cA} z_k(a\mid s_{k,\tau})\widehat{Q}_{k}^{\tilde{\pi}_k}(s_{k,\tau},a)^2\notag\middle|\, \pi^\star,\widehat{P}_{k_i},s_0\right]\\
    \leq&\BE \left[\frac{\eta^2}{2}\sum_{k=k_i}^{k_i+L-1}\sum_{a\in\cA}\tilde{\pi}_k(a\mid s_{k,\tau})\widehat{Q}_{k}^{\tilde{\pi}_k}(s_{k,\tau},a)^2\notag\middle|\, \pi^\star,\widehat{P}_{k_i},s_0\right]\\
    \leq&2\eta^2 L/(1-\gamma)^2\,,
\end{align}
where 
the second line comes from the Young-Fenchel inequality for all $z_k(\cdot\mid s_{k,\tau})\in [\widehat{\pi}_{k+1}(\cdot\mid s_{k,\tau}), \tilde{\pi}_k(\cdot\mid s_{k,\tau})]$ arbitrarily, the third line follows by the first-order optimality condition in Eq. \eqref{eq:OMD_first}, the fourth line is by the mean value theorem of the second derivative for some fixed $\omega_k(\cdot\mid s_{k,\tau})\in[\widehat{\pi}_{k+1}(\cdot\mid s_{k,\tau}), \tilde{\pi}_k(\cdot\mid s_{k,\tau})]$, the fifth line comes from fixing $z_k(\cdot\mid s_{k,\tau})=\omega_k(\cdot\mid s_{k,\tau})$, the sixth line comes from the fact that $z_k(\cdot\mid s_{k,\tau})\leq \tilde{\pi}_k(\cdot\mid s_{k,\tau})$ and the last line is due to that $|\widehat{Q}_{k}^{\tilde{\pi}_k}(s_{k,\tau},a)|\leq 2/(1-\gamma)$.


Substituting Eq. \eqref{eq:eq:OMD_reg2} and Eq. \eqref{eq:eq:OMD_reg3} into Eq. \eqref{eq:eq:OMD_reg1_1}, along with Eq. \eqref{eq:OMD_reg1}, shows that 
\begin{align*}
&\BE\left[\sum_{k=1}^K \left(\widehat{V}_k^{\tilde{\pi}_k}-\widehat{V}_k^{\pi^\star}\right)\right]\\
    =&\BE\left[\BE_{s_0\sim d_0}\left[\sum_{k=1}^K \left(\widehat{V}_k^{\tilde{\pi}_k}(s_0)-\widehat{V}_k^{\pi^\star}(s_0)\right)\right]\right]\\
    =&\BE\left[\BE_{s_0\sim d_0}\left[\sum_{i=1}^N \sum_{k=k_i}^{k_i+L-1} \left(\widehat{V}_k^{\tilde{\pi}_k}(s_0)-\widehat{V}_k^{\pi^\star}(s_0)\right)\right]\right]\\
    =&\sum_{\tau =0}^{\infty}\gamma^{\tau}\BE\left[\BE_{s_0\sim d_0}\left[\sum_{i=1}^N\BE\left[\sum_{k=k_i}^{k_i+L-1}\left\langle \tilde{\pi}_k(\cdot\mid s_{k,\tau})-\pi^\star(\cdot\mid s_{k,\tau}), \widehat{Q}_{k}^{\tilde{\pi}_k}(s_{k,\tau},\cdot)\right\rangle \,\middle|\, \pi^\star,\widehat{P}_{k_i},s_0\right]\right]\right]\\
    \leq&\sum_{\tau =0}^{\infty}\gamma^{\tau}N\cdot\frac{1}{\eta}\left(\ln A+2\eta^2L/(1-\gamma)^2\right)
    \leq\frac{K\sqrt{2\ln A}}{\sqrt{L}(1-\gamma)^2}\,,
\end{align*}
which completes the proof.
\end{proof}

\subsection{Bounding \textsc{Optimism Term}}
We now turn to prove \Cref{lem: optimism}, which provides a (near) optimism at the initial state distribution.
\begin{proof}[Proof of \Cref{lem: optimism}]
Recall $f_{k_i}(s, a)=\left\|\widehat{P}_{k_i}(\cdot \mid s, a)-P^{\star}(\cdot \mid s, a)\right\|_1$.
In the following, we condition on the good event
\begin{align*}
    \cE =\left\{\forall i\in[N], \mathbb{E}_{s \sim \rho_{k_i}, a \sim \bar{\pi}_{k_i}}\left[f_{k_i}^2(s, a)\right] \leq \zeta_{k_i}, \mathbb{E}_{s \sim \rho_{k_i}^\prime, a \sim \bar{\pi}_{k_i}}\left[f_{k_i}^2(s, a)\right] \leq \zeta_{k_i}\,;\right.\\
    \left.\forall i\in[N],\forall \phi,\|\phi(s, a)\|_{\widehat{\Sigma}_{k_i,\phi}^{-1}}=\Theta\left(\|\phi(s, a)\|_{\Sigma_{\rho_{k_i} \times \bar{\pi}_{k_i}, \phi}^{-1}}\right)\right\}\,,
\end{align*}
which is guaranteed to hold with probability $1-\delta$ by using union bound with \Cref{lem: MLE} and \Cref{lem: Concentration of the bonus term for roll-out policy}.

We first consider fixed initial state $s_0$.
For some epoch $i\in[N]$ and episode $k\in\{k_i,k_i+1,\ldots,k_i+L-1\}$, applying Eq. \eqref{eq:simulation_eq1} in \Cref{lem: simulation lemma} implies that
\begin{align}\label{eq:optimism1}
    &\widehat{V}_k^{\pi^\star}(s_0)-V_k^{\pi^\star}(s_0)\notag\\
    =&\frac{1}{1-\gamma}\mathbb{E}_{(s, a) \sim d_{\widehat{P}_{k_i}}^{\pi^\star}}\left[-\widehat{b}_{k_i}(s, a)+\gamma
        \left(\widehat{P}_{k_i}(\cdot\mid s,a)-{P}^\star(\cdot\mid s,a)\right)^\top V_k^{\pi^\star}\right]\notag\\
    \leq&\frac{1}{1-\gamma}\mathbb{E}_{(s, a) \sim d_{\widehat{P}_{k_i}}^{\pi^\star}}\left[-\widehat{b}_{k_i}(s, a)+\frac{\gamma}{1-\gamma}
        f_{k_i}(s,a)\right]\notag\\
    =&\frac{1}{1-\gamma}\mathbb{E}_{(s, a) \sim d_{\widehat{P}_{k_i}}^{\pi^\star}}\left[-\min \left(\alpha_{k_i}\left\|\widehat{\phi}_{k_i}(s, a)\right\|_{\Sigma_{\rho_{k_i} \times \bar{\pi}_{k_i}, \widehat{\phi}_{k_i}}^{-1}}, 2\right)\frac{1}{(1-\gamma)}\right]\notag\\
    &+\frac{\gamma}{(1-\gamma)^2}\left(
    \underbrace{\gamma \mathbb{E}_{(\tilde{s}, \tilde{a}) \sim d_{\widehat{P}_{k_i}}^{\pi^\star}, s \sim \widehat{P}_{k_i}(\cdot\mid\tilde{s}, \tilde{a}), a \sim \pi^\star(s)}\left[f_{k_i}(s,a)\right]}_{\textsc{Term}_1}+\underbrace{ (1-\gamma)\mathbb{E}_{s \sim d_0, a \sim \pi^\star\left(s\right)}\left[f_{k_i}(s,a)\right]}_{\textsc{Term}_2}\right)\,,
\end{align}
where the inequality comes from the Cauchy–Schwarz inequality together with $\|V_k^{\pi^\star}\|_{\infty}=1/(1-\gamma)$, and the second equality is due to 
the fact that $\widehat{\Sigma}_{{k_i}, \widehat{\phi}_{k_i}}$ is an unbiased estimate of $\Sigma_{\rho_{k_i} \times \bar{\pi}_{k_i}, \widehat{\phi}_{k_i}}$ and \Cref{lem:occupancy_measure_expectation_decomp}.

We first bound $\textsc{Term}_2$ as follows:
\begin{align}\label{eq:optimism2}
    (1-\gamma)\mathbb{E}_{s \sim d_0, a \sim \pi^\star\left(s\right)}\left[f_{k_i}(s, a)\right]\leq \sqrt{\frac{(1-\gamma)A}{\mixCoeff} \mathbb{E}_{s \sim \rho_{k_i}, a \sim \bar{\pi}_{k_i}}\left[f_{k_i}^2(s, a)\right]}\leq \sqrt{\frac{(1-\gamma)A\zeta_{k_i}}{ \mixCoeff}}\,,
\end{align}
where the first inequality is due to \Cref{lem: The bound of Initial state term} and the second inequality is by the definition of event $\cE$.

It remains to bound $\textsc{Term}_1$:
\begin{align}\label{eq:optimism3}
&\gamma\mathbb{E}_{(\tilde{s}, \tilde{a}) \sim d_{\widehat{P}_{k_i}}^{\pi^\star}, s \sim \widehat{P}_{k_i}(\cdot\mid\tilde{s}, \tilde{a}), a \sim \pi^\star(s)}\left[f_{k_i}(s,a)\right]\notag \\
\leq& \gamma\mathbb{E}_{(\tilde{s}, \tilde{a}) \sim d_{\widehat{P}_{k_i}}^{\pi^\star}}\left[\left\|\widehat{\phi}_{k_i}(\tilde{s}, \tilde{a})\right\|_{\Sigma_{\rho_{k_i} \times \bar{\pi}_{k_i}, \widehat{\phi}_{k_i}}^{-1}}\right] 
\sqrt{\frac{{k_i}A}{\mixCoeff} \mathbb{E}_{s \sim \rho_{k_i}^{\prime}, a \sim \bar{\pi}_{k_i}}\left[f_{k_i}^2(s, a)\right]+4\lambda_{k_i} d+4{k_i} \zeta_{k_i}}\notag\\
\leq& \gamma\mathbb{E}_{(\tilde{s}, \tilde{a}) \sim d_{\widehat{P}_{k_i}}^{\pi^\star}}\left[\left\|\widehat{\phi}_{k_i}(\tilde{s}, \tilde{a})\right\|_{\Sigma_{\rho_{k_i} \times \bar{\pi}_{k_i}, \widehat{\phi}_{k_i}}^{-1}}\right] \sqrt{\frac{{k_i}A}{\mixCoeff}\zeta_{k_i} +4 \lambda_{k_i} d+4{k_i} \zeta_{k_i}}\notag\\
\lesssim& \alpha_{k_i} \mathbb{E}_{(\tilde{s}, \tilde{a}) \sim d_{\widehat{P}_{k_i}}^{\pi^\star}}\left[\left\|\widehat{\phi}_{k_i}(\tilde{s}, \tilde{a})\right\|_{\Sigma_{\rho_{k_i} \times \bar{\pi}_{k_i}, \widehat{\phi}_{k_i}}^{-1}}\right]\,,
\end{align}
where the first inequality follows by \Cref{lem: one step back estiamte model} as well as $\|f_{k_i}\|_{\infty}\leq 2$, the second inequality is again due to the definition of the good event $\cE$, and the last inequality comes from the definition of $\alpha_{k_i}$.

Now substituting Eq. \eqref{eq:optimism2} and Eq. \eqref{eq:optimism3} into Eq. \eqref{eq:optimism1} shows that 
\begin{align*}
    &\sum_{k=1}^K \left(\widehat{V}_k^{\pi^\star}-V_k^{\pi^\star}\right)\\
    =&\BE_{s_0\sim d_0}\left[\sum_{k=1}^K \left(\widehat{V}_k^{\pi^\star}(s_0)-V_k^{\pi^\star}(s_0)\right)\right]\\
    =&\BE_{s_0\sim d_0}\left[\sum_{i=1}^N \sum_{k=k_i}^{k_i+L-1} \left(\widehat{V}_k^{\pi^\star}(s_0)-V_k^{\pi^\star}(s_0)\right)\right]\\
    \leq&\sum_{i=1}^N \sum_{k=k_i}^{k_i+L-1} \frac{1}{(1-\gamma)^2}\mathbb{E}_{(s, a) \sim d_{\widehat{P}_{k_i}}^{\pi^\star}}\left[-\min \left(\alpha_{k_i}\left\|\widehat{\phi}_{k_i}(s, a)\right\|_{\Sigma_{\rho_{k_i} \times \bar{\pi}_{k_i}, \widehat{\phi}_{k_i}}^{-1}}, 2\right)\right.\\
    &\left.\quad +\min \left(\alpha_{k_i}\left\|\widehat{\phi}_{k_i}(s, a)\right\|_{\Sigma_{\rho_{k_i} \times \bar{\pi}_{k_i}, \widehat{\phi}_{k_i}}^{-1}}+\sqrt{\frac{(1-\gamma)A\zeta_{k_i}}{ \mixCoeff}}, 2\right)\right]\\
    \leq&\sum_{i=1}^N L\sqrt{\frac{A\ln(MN/\delta)}{\mixCoeff(1-\gamma)^3k_i}}\\
    =&\sum_{i=1}^N L\sqrt{\frac{A\ln(MN/\delta)}{\mixCoeff(1-\gamma)^3((i-1)L+1)}}\\
    \leq& (L+\sqrt{K})\sqrt{\frac{A\ln(MN/\delta)}{\mixCoeff(1-\gamma)^3}}\,,
\end{align*}
where the first inequality follows by $\|f_{k_i}\|_{\infty}\leq 2$ for any $i\in[N]$ and the second inequality is due to the definition of $\zeta_{k_i}$ in Lemma \ref{lem: MLE}. 
\end{proof}


\subsection{Bounding \textsc{Estimation Bias Term}}

We now give the proof of \Cref{lem: estimation bias}, which controls the estimation bias term.
\begin{proof}[Proof of \Cref{lem: estimation bias}]
Similar to the proof of \Cref{lem: optimism}, in what follows,
we condition on the good event
\begin{align*}
    \cE =\left\{\forall i\in[N], \mathbb{E}_{s \sim \rho_{k_i}, a \sim \bar{\pi}_{k_i}}\left[f_{k_i}^2(s, a)\right] \leq \zeta_{k_i}, \mathbb{E}_{s \sim \rho_{k_i}^\prime, a \sim \bar{\pi}_{k_i}}\left[f_{k_i}^2(s, a)\right] \leq \zeta_{k_i}\,;\right.\\
    \left.\forall i\in[N],\forall \phi,\|\phi(s, a)\|_{\widehat{\Sigma}_{{k_i},\phi}^{-1}}=\Theta\left(\|\phi(s, a)\|_{\Sigma_{\rho_{k_i} \times \bar{\pi}_{k_i}, \phi}^{-1}}\right)\right\}\,.
\end{align*}

We first consider fixed initial state $s_0$.
For some epoch $i\in[N]$ and episode $k\in\{k_i,k_i+1,\ldots,k_i+L-1\}$,
applying \Cref{lem: simulation lemma} shows that
\begin{align}\label{eq:estimation_bias_1}
&V_k^{\tilde{\pi}_k}(s_0)-\widehat{V}_k^{\tilde{\pi}_k}(s_0)\\
=& (1-\gamma)^{-1} \mathbb{E}_{(s, a) \sim d_{P^{\star}}^{\tilde{\pi}_k}}\left[\widehat{b}_{k_i}(s, a)-\gamma \mathbb{E}_{\widehat{P}_{k_i}\left(s^{\prime} \mid s, a\right)}\left[\widehat{V}^{\tilde{\pi}_k}_k\left(s^{\prime}\right)\right]+\gamma \mathbb{E}_{P^{\star}\left(s^{\prime} \mid s, a\right)}\left[\widehat{V}^{\tilde{\pi}_k}_k\left(s^{\prime}\right)\right]\right]\notag\\
\leq& \mathbb{E}_{(s, a) \sim d_{P^{\star}}^{\tilde{\pi}_k}}\left[\frac{1}{1-\gamma}\widehat{b}_{k_i}(s, a)+\frac{2}{(1-\gamma)^3}f_{k_i}(s,a)\right]\notag\\
=& \underbrace{\gamma \mathbb{E}_{(\tilde{s}, \tilde{a}) \sim d_{P^{\star}}^{\tilde{\pi}_k}, s \sim P^{\star}(\cdot\mid\tilde{s}, \tilde{a}), a \sim \tilde{\pi}_k(s)}\bigg[\frac{1}{1-\gamma}\widehat{b}_{k_i}(s, a)+\frac{2}{(1-\gamma)^3}f_{k_i}(s,a)\bigg]}_{\textsc {Term}_1}\notag\\
&\quad+\underbrace{ (1-\gamma)\mathbb{E}_{s \sim d_0, a \sim \tilde{\pi}_k\left(s_0\right)}\bigg[\frac{1}{1-\gamma}\widehat{b}_{k_i}(s, a)+\frac{2}{(1-\gamma)^3}f_{k_i}(s,a)\bigg]}_{\textsc {Term}_2}\,,
\end{align}
where the inequality follows from the fact that $\|\widehat{b}_{k_i}\|_{\infty} \leq 2/(1-\gamma)$ and
 $\|\widehat{V}^{\tilde{\pi}_k}_{k_i}\|_{\infty} \leq 2 /(1-\gamma)^2$ and the second equality is due to 
 \Cref{lem:occupancy_measure_expectation_decomp}.
$\textsc {Term}_2$ can be bounded as follows:
\begin{align}\label{eq:estimation_bias_2}
    \textsc {Term}_2&=\mathbb{E}_{s \sim d_0, a \sim \tilde{\pi}_k\left(s_0\right)}\left[\widehat{b}_{k_i}(s, a)+\frac{2}{(1-\gamma)^2}f_{k_i}(s,a)\right]\notag\\
    &\leq \sqrt{\frac{A}{(1-\gamma) \mixCoeff} \mathbb{E}_{s \sim \rho_{k_i}, a \sim \bar{\pi}_{k_i}(s)}\left[\widehat{b}_{k_i}^2(s, a)+\frac{4}{(1-\gamma)^4}f_{k_i}^2(s,a)\right]}\notag\\
    &\lesssim\sqrt{\frac{dA\alpha_{k_i}^2}{(1-\gamma)^3 {k_i}\mixCoeff}+\frac{A\zeta_{k_i}}{(1-\gamma)^5\mixCoeff}}\notag\\
    &\leq \sqrt{\frac{dA\alpha_{k_i}^2}{(1-\gamma)^3 {k_i}\mixCoeff}}+\sqrt{\frac{A\zeta_{k_i}}{(1-\gamma)^5\mixCoeff}}\,,
\end{align}
where the first inequality follows from \Cref{lem: The bound of Initial state term}, and the second inequality comes from \Cref{lem: MLE} as well as the following inequality:
\begin{align}\label{eq:est_bias_egx}
    &\mathbb{E}_{s \sim \rho_{k_i}, a \sim \bar{\pi}_{k_i}(s)}\left[{k_i}\widehat{b}_{k_i}^2(s, a)\right]\notag \\
    \leq&\frac{{k_i}\alpha_{k_i}^2}{(1-\gamma)^2}  \mathbb{E}_{s \sim \rho_{k_i}, a \sim \bar{\pi}_{k_i}(s)}\left[\left\|\widehat{\phi}_{k_i}(s, a)\right\|_{\Sigma_{\rho_{k_i} \times \bar{\pi}_{k_i}, \widehat{\phi}_{k_i}}^{-1}}^2\right]\notag\\
    =& \frac{{k_i}\alpha_{k_i}^2}{(1-\gamma)^2} \operatorname{Tr}\left(\mathbb{E}_{s \sim \rho_{k_i}, a \sim \bar{\pi}_{k_i}}\left[\widehat{\phi}_{k_i}(s,a) \widehat{\phi}_{k_i}(s,a)^{\top}\right]\left\{{k_i} \mathbb{E}_{s \sim \rho_{k_i}, a \sim \bar{\pi}_{k_i}}\left[\widehat{\phi}_{k_i}(s,a) \widehat{\phi}_{k_i}(s,a)^{\top}\right]+\lambda_{k_i} I\right\}^{-1}\right)\notag\\
    \leq& \frac{\alpha_{k_i}^2 d}{(1-\gamma)^2}\,,
\end{align}
where the first inequality is because  $\widehat{\Sigma}_{{k_i}, \phi_{k_i}}$ is an unbiased estimate of $\Sigma_{\rho_{k_i} \times \bar{\pi}_{k_i}, \phi_{k_i}}$ and the second inequality is by $\operatorname{Tr}(AB)\geq \operatorname{Tr}(AC)$ for positive semi-definite matrices $A$, $B$, $C$ and $B-C\succeq 0$.

To bound $\textsc {Term}_1$, we note that
\begin{align*}
    &\mathbb{E}_{(\tilde{s}, \tilde{a}) \sim d_{P^{\star}}^{\tilde{\pi}_k}, s \sim P^{\star}(\cdot\mid\tilde{s}, \tilde{a}), a \sim \tilde{\pi}_k(s)}\left[\widehat{b}_{k_i}(s, a)+\frac{2}{(1-\gamma)^2}f_{k_i}(s,a)\right]\\
    \leq& \mathbb{E}_{(\tilde{s}, \tilde{a}) \sim d_{P ^\star}^{\tilde{\pi}_k}}\left[\left\|\phi^{\star}(\tilde{s}, \tilde{a})\right\|_{\Sigma_{\rho_{k_i}, \phi \star}^{-1}}\right] \sqrt{\frac{2{k_i}A}{\mixCoeff\gamma} \mathbb{E}_{s \sim \rho_{k_i}, a \sim \bar{\pi}_{k_i}(s)}\left[\widehat{b}_{k_i}^2(s, a)+\frac{4}{(1-\gamma)^4}f_{k_i}^2(s,a)\right]+\lambda_{k_i}d \frac{36}{(1-\gamma)^4}}\\
    \leq& \mathbb{E}_{(\tilde{s}, \tilde{a}) \sim d_{P ^\star}^{\tilde{\pi}_k}}\left[\left\|\phi^{\star}(\tilde{s}, \tilde{a})\right\|_{\Sigma_{\rho_{k_i}, \phi \star}^{-1}}\right] \sqrt{\frac{2{k_i}A}{\mixCoeff\gamma}\left(\frac{\alpha_{k_i}^2 d}{{k_i}(1-\gamma)^2}+\frac{4}{(1-\gamma)^4}\zeta_{k_i}\right)+\lambda_{k_i}d \frac{36}{(1-\gamma)^4}}\\
    \lesssim&  \mathbb{E}_{(\tilde{s}, \tilde{a}) \sim d_{P ^\star}^{\tilde{\pi}_k}}\left[\left\|\phi^{\star}(\tilde{s}, \tilde{a})\right\|_{\Sigma_{\rho_{k_i}, \phi \star}^{-1}}\right] \sqrt{\frac{dA\alpha_{k_i}^2}{\mixCoeff\gamma(1-\gamma)^2}} 
    + \mathbb{E}_{(\tilde{s}, \tilde{a}) \sim d_{P^\star}^{\tilde{\pi}_k}}\left[\left\|\phi^{\star}(\tilde{s}, \tilde{a})\right\|_{\Sigma_{\rho_{k_i}, \phi \star}^{-1}}\right] \sqrt{\frac{{k_i}A\zeta_{k_i}}{\mixCoeff\gamma(1-\gamma)^4}+\frac{\lambda_{k_i}d}{\gamma(1-\gamma)^4}}\,,
\end{align*}
where the first inequality follows from \Cref{lem: one step back true model}, the AM-GM inequality, $\|\widehat{b}_{k_i}\|_{\infty}\leq 2(1-\gamma)$, and $\|f_{k_i}\|_{\infty}\leq 2$ and the second inequality is due to Eq. \eqref{eq:est_bias_egx}.

The above display implies that $\textsc {Term}_1$ can be bounded as follows:
\begin{align}\label{eq:estimation_bias_3}
    &\gamma \mathbb{E}_{(\tilde{s}, \tilde{a}) \sim d_{P^{\star}}^{\tilde{\pi}_k}, s \sim P^{\star}(\cdot\mid \tilde{s}, \tilde{a}), a \sim \tilde{\pi}_k(s)}\left[\frac{1}{1-\gamma}\widehat{b}_{k_i}(s, a)+\frac{2}{(1-\gamma)^3}f_{k_i}(s,a)\right]\notag\\
    \lesssim& \frac{1}{(1-\gamma)^2}\mathbb{E}_{(\tilde{s}, \tilde{a}) \sim d_{P^\star}^{\tilde{\pi}_k}}\left[\left\|\phi^{\star}(\tilde{s}, \tilde{a})\right\|_{\Sigma_{\rho_{k_i}, \phi^{\star}}^{-1}}\right] \sqrt{\gamma\frac{dA \alpha_{k_i}^2}{\mixCoeff}}
    +\frac{1}{(1-\gamma)^3}\mathbb{E}_{(\tilde{s}, \tilde{a}) \sim d_{P^\star}^{\tilde{\pi}_k}}\left[\left\|\phi^{\star}(\tilde{s}, \tilde{a})\right\|_{\Sigma_{\rho_{k_i}, \phi^{\star}}^{-1}}\right]\notag\\ 
    &\cdot\sqrt{\gamma\left(\frac{ {k_i}A\zeta_{k_i}}{\mixCoeff}+ \lambda_{k_i}d\right)}\notag\\
    \leq& \frac{1}{(1-\gamma)^2}\mathbb{E}_{(\tilde{s}, \tilde{a}) \sim d_{P^\star}^{\tilde{\pi}_k}}\left[\left\|\phi^{\star}(\tilde{s}, \tilde{a})\right\|_{\Sigma_{\rho_{k_i}, \phi^{\star}}^{-1}}\right] \sqrt{\frac{dA \alpha_K^2}{\mixCoeff}} 
    + \frac{1}{(1-\gamma)^3}\mathbb{E}_{(\tilde{s}, \tilde{a}) \sim d_{P^\star}^{\tilde{\pi}_k}}\left[\left\|\phi^{\star}(\tilde{s}, \tilde{a})\right\|_{\Sigma_{\rho_{k_i}, \phi^{\star}}^{-1}} \alpha_K\right]\notag\\
    \lesssim&\frac{1}{(1-\gamma)^3}\mathbb{E}_{(\tilde{s}, \tilde{a}) \sim d_{P^\star}^{\tilde{\pi}_k}}\left[\left\|\phi^{\star}(\tilde{s}, \tilde{a})\right\|_{\Sigma_{\rho_{k_i}, \phi^{\star}}^{-1}}\right] \sqrt{\frac{dA \alpha_K^2}{\mixCoeff}}\,,
\end{align}
where the second inequality is due to that  $\alpha_k=O(\sqrt{\gamma(A/\mixCoeff+d^2)\ln(M k/\delta)})=O(\sqrt{\gamma(A/\mixCoeff+d^2)k\zeta_k})=O(\sqrt{\gamma(Ak\zeta_k/\mixCoeff+\lambda_k d)})=O( \sqrt{\gamma(Ak\zeta_k/\mixCoeff+\lambda_k d+k \zeta_k)})$.



Substituting Eq. \eqref{eq:estimation_bias_2} and Eq. \eqref{eq:estimation_bias_3} into Eq. \eqref{eq:estimation_bias_1}, and taking summation over $k\in[K]$ leads to
\begin{align*}
&\sum_{k=1}^K \left(V_k^{\tilde{\pi}_k}-\widehat{V}_k^{\tilde{\pi}_k}\right)\\
    =&\BE_{s_0\sim d_0}\left[\sum_{k=1}^K \left(V_k^{\tilde{\pi}_k}(s_0)-\widehat{V}_k^{\tilde{\pi}_k}(s_0)\right)\right]\\
    \lesssim&\sum_{i=1}^N \sum_{k=k_i}^{k_i+L-1} \left( \frac{1}{(1-\gamma)^3}\mathbb{E}_{({s}, {a}) \sim d_{P^\star}^{\tilde{\pi}_k}}\left[\left\|\phi^{\star}({s}, {a})\right\|_{\Sigma_{\rho_{k_i}, \phi^{\star}}^{-1}}\right] \sqrt{\frac{dA \alpha_K^2}{\mixCoeff}}
    +\sqrt{\frac{dA\alpha_K^2}{(1-\gamma)^3 {k_i}\mixCoeff}}
    +\sqrt{\frac{A\zeta_{k_i}}{(1-\gamma)^5\mixCoeff}} \right)\\
    \lesssim&\frac{\alpha_K}{(1-\gamma)^3}\sqrt{\frac{dA}{\mixCoeff}}\cdot\sum_{i=1}^N \sum_{k=k_i}^{k_i+L-1}
    \mathbb{E}_{({s}, {a}) \sim d_{P^\star}^{\tilde{\pi}_k}}\left[\left\|\phi^{\star}({s}, {a})\right\|_{\Sigma_{\rho_{k_i}, \phi^{\star}}^{-1}}\right]\\
    \leq&\frac{\alpha_K}{(1-\gamma)^3}\sqrt{\frac{dA}{\mixCoeff}}
    \cdot\sqrt{K\sum_{i=1}^N \sum_{k=k_i}^{k_i+L-1}\mathbb{E}_{({s}, {a}) \sim d_{P^\star}^{\tilde{\pi}_k}}\left[\left\|\phi^{\star}({s}, {a})\right\|^2_{\Sigma_{\rho_{k_i}, \phi^{\star}}^{-1}}\right]}\\
    =&\frac{\alpha_K}{(1-\gamma)^3}\sqrt{\frac{dA}{\mixCoeff}}
    \cdot\sqrt{K\sum_{i=1}^N \sum_{k=k_i}^{k_i+L-1}
    \operatorname{Tr}\left(\mathbb{E}_{({s}, {a}) \sim d_{P^\star}^{\tilde{\pi}_k}}\left[\phi^{\star}({s}, {a})\phi^{\star}({s}, {a})^\top\right]\Sigma_{\rho_{k_i}, \phi^{\star}}^{-1}\right)}\\
    =&\frac{\alpha_K}{(1-\gamma)^3}\sqrt{\frac{dA}{\mixCoeff}}
    \cdot\sqrt{K\sum_{j=1}^L \sum_{i=1}^{N}
    \operatorname{Tr}\left(\mathbb{E}_{({s}, {a}) \sim d_{P^\star}^{\tilde{\pi}_{(i-1)L+j}}}\left[\phi^{\star}({s}, {a})\phi^{\star}({s}, {a})^\top\right]\Sigma_{\rho_{k_i}, \phi^{\star}}^{-1}\right)}\\
    \leq&\frac{\alpha_K}{(1-\gamma)^3}\sqrt{\frac{dA}{\mixCoeff}}\\
    &\cdot\sqrt{K\sum_{j=1}^L \sum_{i=1}^{N}
    \operatorname{Tr}\left(\mathbb{E}_{({s}, {a}) \sim d_{P^\star}^{\tilde{\pi}_{(i-1)L+j}}}\left[\phi^{\star}({s}, {a})\phi^{\star}({s}, {a})^\top\right]\left(\sum_{q=1}^{i-1}\BE_{(s,a)\sim d_{P^\star}^{\tilde{\pi}_{(q-1)L+j}}}\left[\phi^{\star}({s}, {a})\phi^{\star}({s}, {a})^\top\right]+\lambda_1I\right)^{-1}\right)}\\
    \lesssim&\frac{\alpha_K}{(1-\gamma)^3}\sqrt{\frac{dA}{\mixCoeff}}\cdot\sqrt{K\sum_{j=1}^L d\ln \left(1+\frac{N}{d \lambda_1}\right)}\\
    \lesssim&\frac{\alpha_K}{(1-\gamma)^3}\sqrt{\frac{dA}{\mixCoeff}}\cdot\sqrt{d LK\ln \left(1+\frac{K}{d \lambda_1}\right)}\\
    \lesssim&\frac{d^2A\sqrt{KL}}{\mixCoeff(1-\gamma)^3}\sqrt{\ln(1+ K)\ln(M K/\delta)}\,,
\end{align*}
where the third inequality follows from Cauchy–Schwarz inequality together with Jensen's inequality, the fourth inequality is by 
$\operatorname{Tr}(AB)\geq \operatorname{Tr}(AC)$ for positive semi-definite matrices $A$, $B$, $C$ and $B-C\succeq 0$, and
the fifth inequality is due to \Cref{lem:Elliptical potential lemma}. The proof is now completed.
\end{proof}

\subsection{One-step-back Inequalities}
We first present the following lemma, which bounds the quantity under the initial state distribution, for any policy $\pi$. Note that this lemma holds for
any $k\in[K]$.
 \begin{lemma}\label{lem: The bound of Initial state term}
    For any $g: \mathcal{S} \times \mathcal{A} \rightarrow \mathbb{R}$ such that $\|g\|_{\infty} \leq B$ and any policy $\pi$, it holds that
    \begin{align*}
        \mathbb{E}_{s \sim d_0, a \sim \pi\left(s_0\right)}\left[g(s, a)\right]\leq \sqrt{\frac{A}{(1-\gamma) \mixCoeff} \mathbb{E}_{s \sim \rho_k, a \sim \bar{\pi}_k}\left[g^2(s, a)\right]}\,.
    \end{align*}
\end{lemma}
\begin{proof}
    \begin{align*}
    \mathbb{E}_{s \sim d_0, a \sim \pi\left(s_0\right)}\left[g(s,a)\right] 
    &\leq \sqrt{\mathbb{E}_{s \sim d_0, a \sim \pi\left(s_0\right)}\left[g^2(s,a)\right]}\\
    &\leq \sqrt{\max _{(s, a)\in\cS\times\cA} \frac{d_0(s) \pi(a | s)}{\rho_k(s) \bar{\pi}_k(a | s)} \mathbb{E}_{s \sim \rho_k, a \sim \bar{\pi}_k}\left[g^2(s, a)\right]}\\
    &\leq \sqrt{\max _{(s, a)\in\cS\times\cA} \frac{d_0(s) \pi(a | s)}{(1-\gamma)d_0(s) \bar{\pi}_k(a | s)} \mathbb{E}_{s \sim \rho_k, a \sim \bar{\pi}_k}\left[g^2(s, a)\right]}\\
    &\leq  \sqrt{\max _{(s, a)\in\cS\times\cA} \frac{1}{(1-\gamma) \mixCoeff\cdot U(a)} \mathbb{E}_{s \sim \rho_k, a \sim \bar{\pi}_k}\left[g^2(s, a)\right]}\\
    &\leq \sqrt{\frac{A}{(1-\gamma) \mixCoeff} \mathbb{E}_{s \sim \rho_k, a \sim \bar{\pi}_k}\left[g^2(s, a)\right]}\,,
\end{align*}
where the first inequality follows from Jensen's inequality,  the second inequality is by importance sampling, and the fourth inequality is due to the definition of $\bar{\pi}_k$.
\end{proof}

The following lemma shows that 
\begin{align*}
\mathbb{E}_{({s}, {a}) \sim d_{P^{\star}}^{\tilde{\pi}_k}}\left[g(s,a)\right]\lesssim\mathbb{E}_{({s}, {a}) \sim d_{P ^\star}^{\tilde{\pi}_k}}\left[\left\|\phi^{\star}({s}, {a})\right\|_{\Sigma_{\rho_k, \phi \star}^{-1}}\right]\,,
\end{align*} 
if $ \mathbb{E}_{s \sim \rho_k, a \sim \bar{\pi}_k}\left[g^2(s, a)\right]$ is upper bounded.
\begin{lemma}[One-step-back inequality in the true model]\label{lem: one step back true model} For any $g: \mathcal{S} \times \mathcal{A} \to \mathbb{R}$ such that $\|g\|_{\infty} \leq B$, any epoch $i\in[N]$ and any episode $k\in\{k_i,\ldots,k_i+L-1\}$, it holds that 
\begin{align*}
    &\mathbb{E}_{(\tilde{s}, \tilde{a}) \sim d_{P^{\star}}^{\tilde{\pi}_k}, s \sim P^{\star}(\cdot\mid\tilde{s}, \tilde{a}), a \sim \tilde{\pi}_k(s)}\left[g(s,a)\right] \\
    \leq& \mathbb{E}_{(\tilde{s}, \tilde{a}) \sim d_{P ^\star}^{\tilde{\pi}_k}}\left[\left\|\phi^{\star}(\tilde{s}, \tilde{a})\right\|_{\Sigma_{\rho_{k_i}, \phi \star}^{-1}}\right] 
    \sqrt{\frac{{k_i}A}{\mixCoeff\gamma} \mathbb{E}_{s \sim \rho_{k_i}, a \sim \bar{\pi}_{k_i}}\left[g^2(s, a)\right]+\lambda_{k_i}d B^2}\,.
\end{align*}
\end{lemma}
\begin{proof}
To begin with, applying the Cauchy–Schwarz inequality shows that
\begin{align}\label{eq:one_step_back_true}
 &\mathbb{E}_{(\tilde{s}, \tilde{a}) \sim d_{P^{\star}}^{\tilde{\pi}_k}, s \sim P^{\star}(\cdot\mid\tilde{s}, \tilde{a}), a \sim \tilde{\pi}_k(s)}\left[g(s, a)\right]\notag\\
 =&\mathbb{E}_{(\tilde{s}, \tilde{a}) \sim d_{P^{\star}}^{\tilde{\pi}_k}} \left[\phi^{\star}(\tilde{s}, \tilde{a})^{\top} \int \sum_a \mu^{\star}(s) \tilde{\pi}_k(a \mid s) g(s, a) d(s)\right]\notag\\
\leq&\mathbb{E}_{(\tilde{s}, \tilde{a}) \sim d_{P^{\star}}^{\tilde{\pi}_k}}\left[\left\|\phi^{\star}(\tilde{s}, \tilde{a})\right\|_{\Sigma_{\rho_{k_i}, \phi^{\star}}^{-1}}\left\|\int \sum_a \mu^{\star}(s) \tilde{\pi}_k(a \mid s) g(s, a) d(s)\right\|_{\Sigma_{\rho_{k_i}, \phi^{\star}}}\right]\,.
\end{align}
We bound the second quadratic form w.r.t. $\Sigma_{\rho_{k_i}, \phi^{\star}}$ in Eq. \eqref{eq:one_step_back_true} as follows:
\begin{align}\label{ineq: true model before MLE}
& \left\|\int \sum_a \mu^{\star}(s) \tilde{\pi}_k(a \mid s) g(s, a) d(s)\right\|_{\Sigma_{\rho_{k_i}, \phi^{\star}}}^2 \notag\\
=&\left[\int \sum_a \mu^{\star}(s) \tilde{\pi}_k(a \mid s) g(s, a) d(s)\right]^{\top}
\left\{{k_i} \mathbb{E}_{(s, a) \sim \rho_{k_i}}\left[\phi^{\star}(s, a)\phi^{\star}(s, a)^{\top}\right]+\lambda_k I\right\}\notag\\
&\quad\left[\int \sum_a \mu^{\star}(s) \tilde{\pi}_k(a \mid s) g(s, a) d(s)\right] \notag\\
\leq& {k_i} \mathbb{E}_{(\tilde{s}, \tilde{a}) \sim \rho_{k_i}}\left\{\left[\int \sum_a \mu^{\star}(s)^{\top} \phi^{\star}(\tilde{s}, \tilde{a}) \tilde{\pi}_k(a \mid s) g(s, a) d(s)\right]^2\right\}+\lambda_{k_i} d B^2\notag \\
=& {k_i} \mathbb{E}_{(\tilde{s}, \tilde{a}) \sim \rho_{k_i}}\left\{
\BE_{s \sim P^{\star}(\cdot\mid \tilde{s}, \tilde{a}), a \sim \tilde{\pi}_k(s)}\left[ g(s, a)\right]^2\right\}+\lambda_{k_i} d B^2\notag\\
\leq&{k_i}\mathbb{E}_{(\tilde{s}, \tilde{a}) \sim \rho_{k_i}, s \sim P^{\star}(\cdot\mid \tilde{s}, \tilde{a}), a \sim \tilde{\pi}_k(s)}\left[g^2(s, a)\right]+\lambda_{k_i} d B^2\,,
\end{align}
where the first inequality comes from $\|g(s,a)\|_{\infty}\leq B$ together with the regularity condition in \Cref{ass:model_class} that $\left\|\int\mu^\star(s) h(s) \mathrm{d}(s)\right\|_2 \leq \sqrt{d}$ for any $h: \mathcal{S} \rightarrow[0,1]$, and the last inequality follows from the Jensen's inequality.

By importance sampling, it is clear that
\begin{align}\label{ineq: the importance sampling of true model}
 &{k_i}\mathbb{E}_{(\tilde{s},\tilde{a}) \sim \rho_{k_i}, s \sim P^{\star}(\cdot\mid\tilde{s}, \tilde{a}), a \sim \tilde{\pi}_k(s)}\left[g^2(s, a)\right]+\lambda_{k_i} d B^2\notag\\
 \leq& \frac{{k_i}}{\gamma} \mathbb{E}_{s \sim \rho_{k_i}, a \sim \tilde{\pi}_k(s)}\left[g^2(s, a)\right]+\lambda_{k_i} d B^2 \notag\\
\leq&\max_{(s,a)\in\cS\times\cA} \frac{{k_i}}{\gamma}\frac{\tilde{\pi}_k(a \mid s)}{\bar{\pi}_{k_i}(a\mid s)} \mathbb{E}_{s \sim \rho_{k_i}, a \sim \bar{\pi}_{k_i}(s)}\left[g^2(s, a)\right]+\lambda_{k_i} d B^2\notag\\
\leq& \frac{{k_i}A}{\mixCoeff\gamma}\mathbb{E}_{s \sim \rho_{k_i}, a \sim \bar{\pi}_{k_i}(s)}\left[g^2(s, a)\right]+\lambda_{k_i} d B^2\,, 
\end{align}
where the third inequality is due to the definition of $\bar{\pi}_k$, and the first inequality is because 
\begin{align*}
& \gamma \mathbb{E}_{(\tilde{s}, \tilde{a}) \sim \rho_{k_i}, s \sim P^{\star}(\cdot\mid\tilde{s}, \tilde{a}), a \sim \tilde{\pi}_k(s)}\left[g^2(s, a)\right]\\
\leq&\gamma \mathbb{E}_{(\tilde{s}, \tilde{a}) \sim \rho_{k_i}, s \sim P^{\star}(\cdot\mid\tilde{s}, \tilde{a}), a \sim \tilde{\pi}_k(s)}\left[g^2(s, a)\right]+(1-\gamma) \mathbb{E}_{s_0 \sim d_0, a \sim \tilde{\pi}_k(s)}\left[g^2(s, a)\right] \\
=&\mathbb{E}_{s \sim \rho_{k_i}, a \sim \tilde{\pi}_k(s)}\left[g^2(s, a)\right]\,,
\end{align*}
which comes from \Cref{lem:occupancy_measure_expectation_decomp}.

The proof is concluded by substituting Eq. \eqref{ineq: true model before MLE} and Eq. \eqref{ineq: the importance sampling of true model} into Eq. \eqref{eq:one_step_back_true}.
\end{proof}

The following lemma is a counterpart of \Cref{lem: one step back true model}, which shows that 
\begin{align*}
\mathbb{E}_{({s}, {a}) \sim d_{\widehat{P}_k}^{\pi}}\left[g(s,a)\right]\lesssim\mathbb{E}_{({s}, {a}) \sim d_{\widehat{P}_k}^{\pi}}\left[\left\|\widehat{\phi}_k({s}, {a})\right\|_{\Sigma_{\rho_k \times \bar{\pi}_k, \widehat{\phi}_k}^{-1}}\right]\,,
\end{align*}
if $ \mathbb{E}_{s \sim \rho_k^{\prime}, a \sim \bar{\pi}_k}\left[g^2(s, a)\right]$ is upper bounded. Note that compared with  \Cref{lem: one step back true model}, this lemma additionally needs to condition on the event that the MLE guarantee (\textit{cf.}, \Cref{lem: MLE}) holds.
\begin{lemma}[One-step-back inequality in the learned model]\label{lem: one step back estiamte model} 
Conditioned on the event where the MLE guarantee in \Cref{lem: MLE} holds, \textit{i.e.}, $\mathbb{E}_{s \sim \rho_{k_i}, a \sim \bar{\pi}_{k_i}}\left[f_{k_i}(s, a)^2\right] \lesssim \zeta_{k_i}$, for any epoch $i\in[N]$.
Then for any $g:\mathcal{S} \times \mathcal{A} \to \mathbb{R}$ such that $\|g\|_{\infty} \leq B$, any epoch $i\in[N]$ and any policy $\pi$, it holds that
\begin{align*}
    &\mathbb{E}_{(\tilde{s}, \tilde{a}) \sim d_{\widehat{P}_{k_i}}^{\pi}, s \sim \widehat{P}_{k_i}(\cdot\mid\tilde{s}, \tilde{a}), a \sim \pi(s)}\left[g(s,a)\right]\\
    \leq& \mathbb{E}_{(\tilde{s}, \tilde{a}) \sim d_{\widehat{P}_{k_i}}^\pi}\left[\left\|\widehat{\phi}_{k_i}(\tilde{s}, \tilde{a})\right\|_{\Sigma_{\rho_{k_i} \times \bar{\pi}_{k_i}, \widehat{\phi}_{k_i}}^{-1}}\right]
    \sqrt{\frac{{k_i}A}{\mixCoeff} \mathbb{E}_{s \sim \rho_{k_i}^{\prime}, a \sim \bar{\pi}_{k_i}}\left[g^2(s, a)\right]+B^2 \lambda_{k_i} d+{k_i} B^2 \zeta_{k_i}}\,,
\end{align*}
where recall that $\bar{\pi}_k(\cdot\mid s)=\mixCoeff\cdot U(\cA)+(1-\mixCoeff)\cdot 1/k\sum_{j=1}^{k}\tilde{\pi}_{j}(\cdot\mid s)$.
\end{lemma}
\begin{proof}
The proof of this lemma is generally similar to that of \Cref{lem: one step back true model}. We start by applying the Cauchy–Schwarz inequality:
\begin{align}\label{eq:one_step_back_learned1}
 &\mathbb{E}_{(\tilde{s}, \tilde{a}) \sim d_{\widehat{P}_{k_i}}^{\pi}, s \sim \widehat{P}_{k_i}(\cdot\mid\tilde{s}, \tilde{a}), a \sim \pi(s)}\left[g(s, a)\right]\notag\\
 =&\mathbb{E}_{(\tilde{s}, \tilde{a}) \sim d_{\widehat{P}_{k_i}}^{\pi}} \left[\widehat{\phi}_{k_i}(\tilde{s}, \tilde{a})^{\top} \int \sum_a \widehat{\mu}_{k_i}(s) \pi(a \mid s) g(s, a) d(s)\right]\notag\\
\leq& \mathbb{E}_{(\tilde{s}, \tilde{a}) \sim d_{\widehat{P}_{k_i}}^{\pi}}
\left[\left\|\widehat{\phi}_{k_i}(\tilde{s}, \tilde{a})\right\|_{\Sigma_{\rho_{k_i} \times \bar{\pi}_{k_i}, \widehat{\phi}_{k_i}}^{-1}}\left\|\int \sum_a \widehat{\mu}_{k_i}(s) \pi(a \mid s) g(s, a) d(s)\right\|_{\Sigma_{\rho_{k_i} \times \bar{\pi}_{k_i}, \widehat{\phi}_{k_i}}}\right]\,.
\end{align}

We now bound the second quadratic form w.r.t. $\Sigma_{\rho_{k_i} \times \bar{\pi}_{k_i}, \widehat{\phi}_{k_i}}$ in Eq. \eqref{eq:one_step_back_learned1} as follows:
\begin{align}\label{eq:one_step_back_learned2}
& \left\|\int \sum_a \widehat{\mu}_{k_i}(s) \pi(a \mid s) g(s, a) d(s)\right\|_{\Sigma_{\rho_{k_i} \times \bar{\pi}_{k_i}, \widehat{\phi}_{k_i}}}^2 \notag\\
=&\left[\int \sum_a \widehat{\mu}_{k_i}(s) \pi(a \mid s) g(s, a) d(s)\right]^{\top}\cdot
\left\{k_i \mathbb{E}_{s \sim \rho_{k_i}, a \sim \bar{\pi}_{k_i}}\left[\widehat{\phi}_{k_i}(s,a) \widehat{\phi}_{k_i}(s,a)^{\top}\right]+\lambda_{k_i} I\right\}\notag\\
&\cdot\left[\int \sum_a \widehat{\mu}_{k_i}(s) \pi(a \mid s) g(s, a) d(s)\right]\notag\\
\leq& {k_i} \mathbb{E}_{\tilde{s} \sim \rho_{k_i}, \tilde{a} \sim \bar{\pi}_{k_i}}\left\{\left[\int \sum_a \widehat{\mu}_{k_i}(s)^{\top} \widehat{\phi}_{k_i}(\tilde{s}, \tilde{a}) \pi(a \mid s) g(s, a) d(s)\right]^2\right\}+B^2 \lambda_{k_i} d\notag\\
=&{k_i} \mathbb{E}_{\tilde{s} \sim \rho_{k_i}, \tilde{a} \sim \bar{\pi}_{k_i}}\left\{\mathbb{E}_{s \sim \widehat{P}_{k_i}(\cdot\mid\tilde{s}, \tilde{a}), a \sim \pi(s)}[g(s, a)]^2\right\}+B^2 \lambda_{k_i} d\,,
\end{align}
where the first inequality is because $\|g(s,a)\|_{\infty}\leq B$ as well as the regularity condition in \Cref{ass:model_class} that $\left\|\int\mu(s) h(s) \mathrm{d}(s)\right\|_2 \leq \sqrt{d}$ for any $h: \mathcal{S} \rightarrow[0,1]$ and any $\mu\in\Psi$.

Moreover, using the MLE guarantee in \Cref{lem: MLE}, we have that
\begin{align}\label{eq:one_step_back_learned3}
&{k_i} \mathbb{E}_{\tilde{s} \sim \rho_{k_i}, \tilde{a} \sim \bar{\pi}_{k_i}}\left\{\mathbb{E}_{s \sim \widehat{P}_{k_i}(\cdot\mid\tilde{s}, \tilde{a}), a \sim \pi(s)}[g(s, a)]^2\right\}+B^2 \lambda_{k_i} d\notag\\
\leq& {k_i} \mathbb{E}_{\tilde{s} \sim \rho_{k_i}, \tilde{a} \sim \bar{\pi}_{k_i}}\left\{\mathbb{E}_{s \sim P^\star(\cdot\mid\tilde{s}, \tilde{a}), a \sim \pi(s)}[g(s, a)]^2\right\}+B^2 \lambda_{k_i} d+{k_i} B^2 \zeta_{k_i} \notag\\
\leq& {k_i} \mathbb{E}_{\tilde{s} \sim \rho_{k_i}, \tilde{a} \sim \bar{\pi}_{k_i}, s \sim P^{\star}(\cdot\mid\tilde{s}, \tilde{a}), a \sim \pi(s)}\left[g^2(s, a)\right]+B^2 \lambda_{k_i} d+B^2 {k_i} \zeta_{k_i}\notag\\
\leq& \frac{{k_i}A}{\mixCoeff}\mathbb{E}_{\tilde{s} \sim \rho_{k_i}, \tilde{a} \sim \bar{\pi}_{k_i}, s \sim P^\star(\cdot\mid\tilde{s}, \tilde{a}), a \sim \bar{\pi}_{k_i}}\left[g^2(s, a)\right]+B^2 \lambda_{k_i} d+B^2 {k_i} \zeta_{k_i} \notag\\
\leq& \frac{{k_i}A}{\mixCoeff} \mathbb{E}_{s \sim \rho_{k_i}^{\prime}, a \sim \bar{\pi}_{k_i}}\left[g^2(s, a)\right]+B^2 \lambda_{k_i} d+B^2 {k_i} \zeta_{k_i} \,,
\end{align}
where the second inequality follows by Jensen's inequality, the third inequality comes from importance sampling and the definition of $\bar{\pi}_{k_i}$, and the last inequality is due to the definition of $\rho_{k_i}^{\prime}$. 

The proof is now concluded by substituting Eq. \eqref{eq:one_step_back_learned2} and Eq. \eqref{eq:one_step_back_learned3} into Eq. \eqref{eq:one_step_back_learned1}.
\end{proof}

The following lemma shows that the expectation of any state-action function $g:\cS\times\cA\to\BR$ w.r.t. $d_{P}^{\pi_1}$ and $\pi_2$ can be decomposed into (a) the one-step-back expectation of $g$ w.r.t. $d_{P}^{\pi_1}$ and $\pi_2$; and (b) the expectation of $g$ w.r.t. $d_0$ and $\pi_2$.
\begin{lemma}\label{lem:occupancy_measure_expectation_decomp}
For any $P$, and any policy $\pi_1$ and $\pi_2$, it holds that
\begin{align*}
    &\mathbb{E}_{s \sim d_{P}^{\pi_1},a\sim\pi_2(\cdot\mid s)}[g(s, a)]\\
    =&\gamma \mathbb{E}_{(\tilde{s}, \tilde{a}) \sim d_{P}^{\pi_1}, s \sim P(\cdot\mid \tilde{s}, \tilde{a}), a \sim \pi_2(\cdot\mid s)}[g(s, a)]+(1-\gamma) \mathbb{E}_{s \sim d_0, a \sim \pi_2\left(\cdot\mid s_0\right)}[g(s, a)]\,.
\end{align*}
\end{lemma}
\begin{proof}
\begin{align*}
    &\mathbb{E}_{s \sim d_{P}^{\pi_1},a\sim\pi_2(\cdot\mid s)}[g(s, a)]\\
    =&\sum_{t=1}^{\infty}(1-\gamma)\gamma^t\mathbb{E}_{s \sim d_{P,t}^{\pi_1},a\sim\pi_2(\cdot\mid s)}[g(s, a)]+
    (1-\gamma)\mathbb{E}_{s \sim d_{P,0}^{\pi_1},a\sim\pi_2(\cdot\mid s)}[g(s, a)]\\
    =&\gamma\sum_{t=0}^{\infty}(1-\gamma)\gamma^t\mathbb{E}_{s \sim d_{P,t+1}^{\pi_1},a\sim\pi_2(\cdot\mid s)}[g(s, a)]+
    (1-\gamma) \mathbb{E}_{s \sim d_0, a \sim \pi_2\left(\cdot\mid s_0\right)}[g(s, a)]\\
    =&\gamma\sum_{t=0}^{\infty}(1-\gamma)\gamma^t\mathbb{E}_{(\tilde{s},\tilde{a}) \sim d_{P,t}^{\pi_1},s\sim P(\cdot\mid \tilde{s},\tilde{a}),a\sim\pi_2(\cdot\mid s)}[g(s, a)]+
    (1-\gamma) \mathbb{E}_{s \sim d_0, a \sim \pi_2\left(\cdot\mid s_0\right)}[g(s, a)]\\
    =&\gamma \mathbb{E}_{(\tilde{s}, \tilde{a}) \sim d_{P}^{\pi_1}, s \sim P(\cdot\mid \tilde{s}, \tilde{a}), a \sim \pi_2(\cdot\mid s)}[g(s, a)]+(1-\gamma) \mathbb{E}_{s \sim d_0, a \sim \pi_2\left(\cdot\mid s_0\right)}[g(s, a)]\,.
\end{align*}
\end{proof}

\section{Omitted Analysis of The Regret Lower Bound}\label{app:sec:lb}

In this section, we provide the proof of Theorem \ref{thm:lb}.
For the remainder of this section, we switch from loss functions to reward functions for convenience since we now consider MDPs with fixed loss functions.

\subsection{Construction of Hard-to-learn MDP Instances}
\begin{figure}[ht]
\centering

\begin{tikzpicture}[->,>=stealth',shorten >=1pt,auto,node distance=2.8cm,semithick]
\tikzstyle{every state}=[fill=white,text=black]
\node[state,minimum width =35pt,minimum height =35pt,draw=blue!75,fill=blue!20]        (A)     {$s_{1,1}$};
\node[state,minimum width =35pt,minimum height =35pt,draw=blue!75,fill=blue!20]        (S21)[below left=1.5cm and 4.5cm of A]    {$s_{2,1}$}; 
\node[state,minimum width =35pt,minimum height =35pt,draw=blue!75,fill=blue!20]        (S2l)[below left=1.5cm and 0.5cm of A]    {$s_{2,i^\star}$};
\node[state,minimum width =35pt,minimum height =35pt,draw=blue!75,fill=blue!20]        (S2dm4)[below right=1.5cm and 1.75cm of A]    {$s_{2,d-4}$};
\node[state,minimum width =35pt,minimum height =35pt,draw=blue!75,fill=blue!20]        (So)[below right=1.5cm  and 4.5cm of A]    {$s^o$}; 
\path 
	(A) edge (S21)
	(A) edge (S2l)
        (A) edge (S2dm4)
        (A) edge (So);   
\path[-]
        (S21)  edge[loosely dotted]   (S2l)
	(S2l)  edge[loosely dotted]   (S2dm4);
\node[state,minimum width =35pt,minimum height =35pt,draw=blue!75,fill=blue!20]     (Sg)[below  left =2.5cm and 1cm of S2l]    {$s^g$}; 
\node[state,minimum width =35pt,minimum height =35pt,draw=blue!75,fill=blue!20]     (Sb)[below  right =2.5cm and 1cm of S2l]    {$s^b$};
\path 
    (S21)  edge[dashed]   node[above left= 0.1cm and 0.3cm]{$\frac{1}{2}$} (Sg)
    (S21)  edge[dashed]   node[above left=0.1cm and 1.1cm]{$\frac{1}{2}$} (Sb)
    (S2l)  edge[dashed]   node[below]{} (Sg)
    (S2l)  edge[dashed]   node[below]{} (Sb)
    (S2l)  edge[blue, dashed, bend right]   node[above= 0.5cm]{$\textcolor{blue}{\frac{1}{2}+\varepsilon} $} (Sg)
    (S2l)  edge[blue, dashed, bend left]   node[above=0.5cm]{$\textcolor{blue}{\frac{1}{2}-\varepsilon} $} (Sb)
    (S2dm4)  edge[dashed]   node[below]{} (Sg)
    (S2dm4)  edge[dashed]   node[below]{} (Sb)
    (Sg)  edge [loop below] node{$1$} (Sg)
    (Sb)  edge [loop below] node{$1$}  (Sb);
\node[state, fill=white,draw=none]    (RG) [below left=0.01cm and 0.1cm of Sg] {$r(s^g, a)={1}$};
\node[state, fill=white,draw=none]    (RB) [below left=0.01cm and 0.1cm of Sb] {$r(s^b, a)=0$};
\node[state,minimum width =35pt,minimum height =35pt,draw=blue!75,fill=blue!20]     (So1)[below  left =2.5cm and 1.5cm of So]    {$s^o_{1}$}; 
\node[state,minimum width =35pt,minimum height =35pt,draw=blue!75,fill=blue!20]     (Soj)[below =2.1cm of So]    {$s^o_{j}$}; 
\node[state,minimum width =35pt,minimum height =35pt,draw=blue!75,fill=blue!20]     (Soom1)[below  right =2.5cm and 1.5cm of So]    {$s_{S-d}^o$}; 
\node (rect) [minimum width=0.1cm,minimum height=1cm] (ROS)[below  left =0.25cm and 2.1cm of Soj] {$r(s^o_{1}, a)=\frac{1}{2}$};
\path 
	(So) edge[dashed] (So1)
	(So) edge[dashed] (Soj)
        (So) edge[dashed] (Soom1);
\path[-]
	(So1) edge[loosely dotted] (Soj)
	(Soj) edge[loosely dotted] (Soom1);
\path
        (So1)  edge[dashed,loop below] node{$\frac{1}{S-d}$} (So1)
        (So1)  edge[dashed, bend right] node[below=0.0001cm]{$\frac{1}{S-d}$} (Soj)
        (So1)  edge[dashed, bend right=55, looseness=1] node[below = 0.0001cm]{$\frac{1}{S-d}$} (Soom1)
        (Soj)  edge [dashed,bend right] node{} (So1)
        (Soj)  edge [dashed,loop below] node{} (Soj)
        (Soj)  edge [dashed,bend left] node{} (Soom1)
        (Soom1)  edge [dashed,bend right=50, looseness=1] node{} (So1)
        (Soom1)  edge [dashed,bend left] node{} (Soj)
        (Soom1)  edge [dashed,loop below] node{} (Soom1);
\end{tikzpicture}
\caption{The class of the hard-to-learn low-rank MDP instances used in the proof of Theorem \ref{thm:lb}.}\label{fig:lower_bound_instance}
\end{figure}
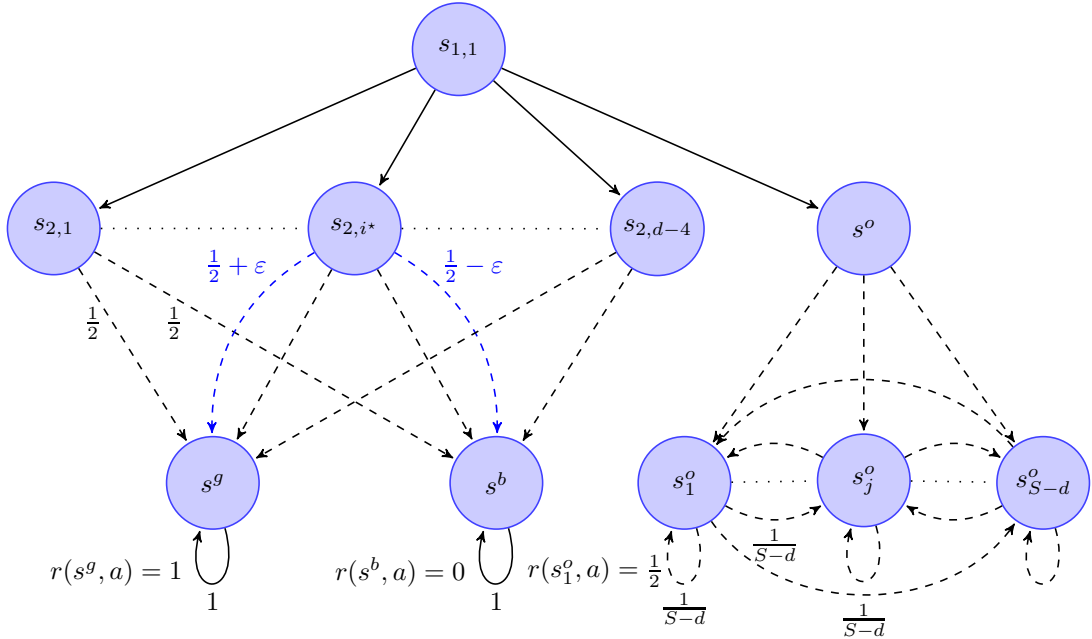\noindent

To prove our regret lower bound in Theorem \ref{thm:lb}, we construct a class of hard-to-learn low-rank MDP instances, as shown in Figure \ref{fig:lower_bound_instance}.
To begin with, we first introduce the reference low-rank MDP $\cM_0$, with its elements detailed as follows:
\begin{itemize}
    \item State space: $\cS=\{s_{1,1},s_{2,1},s_{2,2},\ldots,s_{2,d-4}\}\cup\{s^g,s^b\}\cup\{s^o\}\cup\cS_{\cO}$, where $\cS_{\cO}=\{s^o_{i}\}_{i=1}^{S-d}$ denotes the set of ``outlier states'', $s^g$ denotes the ``good state'', and $s^b$ denotes the ``bad state''.
    \item Action space: $\cA=\{a_1,a_2,\ldots,a_{A}\}$.
    \item Reward function: $r(s,a)=\BI\{s=s^g\}+\frac{1}{2}\BI\{s\in\cS_{\cO}\}$.
    \item Transitions: 
    \begin{itemize}
        \item For the initial state $s_{1,1}$, the learner will deterministically transit to state $s_{2,i}$ if taking action $a_i$, $\forall i\in[d-4]$, and will transit to state $s^o$ otherwise. Formally, $P^\star\left(s_{2,i} \mid s_{1,1}, a_i\right)=1$, $\forall i\in[d-4]$, and $P^\star\left(s^o \mid s_{1,1}, a_i\right)=1$, $\forall i\in[A]\setminus[d-4]$.
        \item For state $s_{2,i}\in\{s_{2,1},s_{2,2},\ldots,s_{2,d-4}\}$, the learner will transit to good state $s^{g}$ and bad state $s^b$ uniformly at random, no matter what action it takes, $\textit{i.e.}$, $P^\star\left(s^{g} \mid s_{2,i}, a\right)=P^\star\left(s^{b} \mid s_{2,i}, a\right)=\frac{1}{2}$, $\forall i\in[d-4]$ and $a\in\cA$.
        \item For states $s^o$ and $s^o_{i}\in\cS_{\cO}$, the learner will uniformly transit to a state $s^o_{j}\in\cS_{\cO}$, no matter what action it takes. Formally, $P^\star\left(s^o_{j} \mid s^o, a\right)=P^\star\left(s^o_{j} \mid s^o_{i}, a\right)=\frac{1}{S-d}$, $\forall s^o_{i},s^o_{j}\in\cS_{\cO}$ and $a\in\cA$.
        \item For states $s^g$ and $s^b$, the learner will stay at the current state no matter what action it takes, which means that $P^\star\left(s^{g} \mid s^g, a\right)=P^\star\left(s^{b} \mid s^{b}, a\right)=1$.
    \end{itemize}
\end{itemize}
Further, the transitions of the above MDP can be realized by $P^\star(s^\prime\mid s,a)=\langle\phi^\star(s,a),\mu^\star(s^\prime)\rangle$, with the following features, which thus implies that this MDP is indeed a low-rank MDP:
\begin{align*}
    &\phi^\star(s_{1,1},a_i)=\bm{e}_i, \quad \mu^\star(s_{2,i})=\bm{e}_i, \,\forall i\in[d-4],\quad \mu^\star(s_{1,1})=\bm{0}\\
    &\phi^\star(s_{1,1},a_i)=(0,\ldots, 0,1,0),\,\forall i\in[A]\setminus[d-4],\quad\mu^\star(s^o)=(0,\ldots, 0,1,0)\\
    &\phi^\star\left(s_{2,j}, a\right)=(0, \ldots, 0, \frac{1}{2}, \frac{1}{2}, 0,0),\,\forall a\in \cA,\quad\mu^\star(s^g)=(0, \ldots, 0,1,0,0,0),\quad\mu^\star(s^b)=(0, \ldots, 0,0,1,0,0)\\
    &\phi^\star(s^o,a)=\phi^\star(s^o_{j},a)=(0,\ldots,0,\frac{1}{S-d}),\,\forall a\in \cA,\quad \mu^\star(s^o_{j})=(0,\ldots,0,1)\\
    &\phi^\star(s^g,a)=\mu^\star(s^g),\quad \phi^\star(s^b,a)=\mu^\star(s^b),\,\forall a\in \cA\,.
\end{align*}
Based on the reference MDP $\cM_0$, we define other low-rank MDP instances $\cM_{(i^\star,a^\star)}$, $\forall (i^\star,a^\star)\in[d-4]\times \cA$. In specific, the only difference between $\cM_{(i^\star,a^\star)}$ and $\cM_0$ is that $\phi^\star(s_{2,i^\star},a^\star)=(0, \ldots, 0, \frac{1}{2}+\varepsilon, \frac{1}{2}-\varepsilon, 0,0)$, such that $P^\star\left(s_g \mid s_{2,i^\star}, a^\star\right)=\frac{1}{2}+\varepsilon$, and $P^\star\left(s_b \mid s_{2,i^\star}, a^\star\right)=\frac{1}{2}-\varepsilon$, for some $\varepsilon>0$ to be defined later.

\subsection{Proof of Theorem \ref{thm:lb}}
Based on the class of hard-to-learn low-rank MDP instances constructed above, we are now ready to prove the regret lower bound in Theorem \ref{thm:lb}.
\begin{proof}[Proof of Theorem \ref{thm:lb}]

In what follows, we denote by $\mathbb{P}_{\left(i^\star, a^\star\right)} \coloneqq \mathbb{P}_{\operatorname{Alg}, \mathcal{M}_{\left(i^\star, a^\star\right)}}$ the probability measure over the outcomes induced by the interaction between $\operatorname{Alg}$ and $\mathcal{M}_{\left(i^\star, a^\star\right)}$, and by $\mathbb{E}_{\left(i^\star, a^\star\right)} \coloneqq \mathbb{E}_{\operatorname{Alg}, \mathcal{M}_{\left(i^\star, a^\star\right)}}$ the expectation with respect to $\mathbb{P}_{\left(i^\star, a^\star\right)}$.
\paragraph{Regret of $\operatorname{Alg}$ in $\mathcal{M}_{\left(i^\star, a^\star\right)}$} 
For some $\cM_{(i^\star,a^\star)}$, its optimal policy $\pi^\star_{(i^\star,a^\star)}:\cS\to\cA$ satisfies that $\pi^\star_{(i^\star,a^\star)}(s_{1,1})=a_{i^\star}$ and $\pi^\star_{(i^\star,a^\star)}(s_{2,i^\star})=a^\star$, with the optimal value function 
\begin{align}\label{eq:lb_eq1}
V_{0}^\star(s_{1,1}) & =\mathbb{E}\left[\sum_{\tau=0}^{+\infty} \gamma^\tau r(s_{\tau}, a_\tau) \mid \pi^\star_{(i^\star,a^\star)},P^\star_{(i^\star,a^\star)}, s_0=s_{1,1}\right]=\sum_{\tau=2}^{+\infty} \gamma^\tau \left(\frac{1}{2}+\varepsilon\right)=\frac{\gamma^2}{1-\gamma}\left(\frac{1}{2}+\varepsilon\right)\,.
\end{align}
For some policy $\pi$, it is also clear that its value function satisfies
\begin{align}\label{eq:lb_eq2}
V_{0}^\pi(s_{1,1}) &=\frac{\gamma^2}{1-\gamma}\left(\frac{1}{2}+\varepsilon\BP_{\left(i^\star, a^\star\right)}\left((s_2,a_2)=(s_{2,i^\star},a^\star)\right)\right)\,.
\end{align}
Combining Eq. \eqref{eq:lb_eq1} and \eqref{eq:lb_eq2} shows that the regret of $\operatorname{Alg}$ in $\cM_{(i^\star,a^\star)}$ satisfies
\begin{align*}
\cR_K(\operatorname{Alg}, \cM_{(i^\star,a^\star}))&=\frac{\gamma^2\varepsilon}{1-\gamma}K\left(1-\frac{1}{K}\BE_{(i^\star,a^\star)}\left[\sum_{k=1}^K\BI\{(s^k_2,a^k_2)=(s_{2,i^\star},a^\star)\} \right]\right)\\
    &=\frac{\gamma^2\varepsilon}{1-\gamma}K\left(1-\frac{1}{K}\BE_{(i^\star,a^\star)}\left[N^K_{(i^\star,a^\star)} \right]\right)\,,
\end{align*}
where we define $N^K_{(i^\star,a^\star)}\coloneqq \sum_{k=1}^K\BI\{(s^k_2,a^k_2)=(s_{2,i^\star},a^\star)\}$.

\paragraph{Maximum Regret of $\operatorname{Alg}$ over All Possible $\cM_{(i^\star,a^\star)}$}
With $\cR_K(\operatorname{Alg}, \cM_{(i^\star,a^\star}))$ in the above equation, we can deduce that 
\begin{align}\label{eq:lb_eq3}
    \max_{(i^\star,a^\star)}\cR_K(\operatorname{Alg},\cM_{(i^\star,a^\star)})&\geq \frac{1}{(d-4)A}\sum_{(i^\star,a^\star)}\cR_K(\operatorname{Alg},\cM_{(i^\star,a^\star}))\notag\\
    &\geq \frac{\gamma^2\varepsilon}{1-\gamma}K\left(1-\frac{1}{K(d-4)A}\sum_{(i^\star,a^\star)}\BE_{(i^\star,a^\star)}\left[N^K_{(i^\star,a^\star)} \right]\right)\,.
\end{align}
To lower bound the above display, it remains to upper bound $\sum_{(i^\star,a^\star)}\BE_{(i^\star,a^\star)}\left[N^K_{(i^\star,a^\star)} \right]$. To this end, by Lemma 1 in the work of \citet{GarivierMS19} together with the fact that $N^K_{(i^\star,a^\star)}/K\in[0,1]$, it holds that 
\begin{align*}
    \operatorname{KL}\left(\operatorname{Ber}\left(\frac{1}{K}\BE_0\left[N^K_{(i^\star,a^\star)}\right]\right),\operatorname{Ber}\left(\frac{1}{K}\BE_{(i^\star,a^\star)}\left[N^K_{(i^\star,a^\star)}\right]\right)\right)\leq \operatorname{KL}\left(\BP_0,\BP_{(i^\star,a^\star)}\right)\,.
\end{align*}
This implies that 
\begin{align*}
    \frac{1}{K}\BE_{(i^\star,a^\star)}\left[N^K_{(i^\star,a^\star)}\right]
    &\leq \frac{1}{K}\BE_0 \left[N^K_{(i^\star,a^\star)}\right]+\sqrt{\frac{1}{2}\operatorname{KL}\left(\BP_0,\BP_{(i^\star,a^\star)}\right)}\\
    &=\frac{1}{K}\BE_0 \left[N^K_{(i^\star,a^\star)}\right]+\varepsilon\sqrt{2}\sqrt{\BE_0\left[N^K_{(i^\star,a^\star)}\right]}\,,
\end{align*}
where the inequality is due to Pinsker’s inequality that $(p-q)^2 \leq \frac{1}{2} \operatorname{KL}(\operatorname{Ber}(p), \operatorname{Ber}(q))$, for $p,q\in[0,1]$, and the equality comes from Lemma 15.1 of \citet{lattimore2020bandit} and Lemma 14 of \citet{DominguesMKV21} as well as assuming $0\leq\varepsilon\leq\frac{1}{4}$. 

Based on this, one can see that 
\begin{align}\label{eq:lb_eq4}
    \frac{1}{K}\sum_{(i^\star,a^\star)}\BE_{(i^\star,a^\star)}\left[N^K_{(i^\star,a^\star)}\right]
    &\leq \frac{1}{K}\sum_{(i^\star,a^\star)}\BE_0 \left[N^K_{(i^\star,a^\star)}\right]+\varepsilon\sqrt{2}\sum_{(i^\star,a^\star)}\sqrt{\BE_0\left[N^K_{(i^\star,a^\star)}\right]}\notag\\
    &\leq 1+\varepsilon\sqrt{2}\sqrt{(d-4)AK}\,,
\end{align}
where the second inequality follows from using the Cauchy-Schwartz
inequality together with the fact that $N^K_{(i^\star,a^\star)}\leq K$.

\paragraph{Optimizing $\varepsilon$ to Lower Bound the Maximum Regret} 
Substituting Eq. \eqref{eq:lb_eq4} into Eq. \eqref{eq:lb_eq3} leads to 
\begin{align*}
    \max_{(i^\star,a^\star)}\cR_K(\operatorname{Alg},\cM_{(i^\star,a^\star)})
    &\geq\frac{\gamma^2\varepsilon}{1-\gamma}K\left(1-\frac{1}{(d-4)A}-\varepsilon\sqrt{2}\sqrt{\frac{K}{(d-4)A}}\right)\\
    &\geq\frac{1}{4\sqrt{2}}\cdot\frac{\gamma^2}{1-\gamma}\left(1-\frac{1}{(d-4)A}\right)^2\sqrt{(d-4)AK}\\
    &\geq \frac{361}{1600\sqrt{2}}\cdot\frac{\gamma^2}{1-\gamma}\sqrt{(d-4)AK}\,,
\end{align*}
where the second inequality comes from by choosing $\varepsilon=\frac{1}{2\sqrt{2}}\left(1-\frac{1}{(d-4)A}\right)\sqrt{\frac{(d-4)A}{K}}$ and the last inequality is due to $d\geq 8$ and $A\geq d-3$. Finally, note that 
$\varepsilon\leq \frac{1}{4}$ is guaranteed when $K\geq 2(d-4)A$. The proof is thus concluded.
\end{proof}

\section{Auxiliary Lemmas} 
We first introduce the concentration of MLE, the i.i.d. version of which at least dates back to Chapter 7 of \citet{geer2000empirical} and the non-i.i.d. version of which is first proved by \citet{Flambe20} and also appears in the analysis of  \citet{REPUCB22}.
\begin{lemma}[MLE guarantee]\label{lem: MLE}
For some fixed epoch $i\in[N]$, with probability $1-\delta$, it holds that
\begin{align*}
    \mathbb{E}_{s \sim\left\{0.5 \rho_{k_i}+0.5 \rho_{k_i}^{\prime}\right\}, a \sim \bar{\pi}_{k_i}(s)}\left[\left\|\widehat{P}_{k_i}(\cdot \mid s, a)-P^{\star}(\cdot\mid s, a)\right\|_1^2\right] \lesssim \zeta\,, \quad \zeta\coloneqq\frac{\ln (M / \delta)}{{k_i}}\,.
\end{align*}
Therefore, simultaneously for all epoch $ i \in [N]$, with probability $1-\delta$, it holds that
\begin{align*}
\mathbb{E}_{s \sim\left\{0.5 \rho_{k_i}+0.5 \rho_{k_i}^{\prime}\right\}, a \sim \bar{\pi}_{k_i}(s)}\left[\left\|\widehat{P}_{k_i}(\cdot \mid s, a)-P^{\star}(\cdot \mid s, a)\right\|_1^2\right] \lesssim\zeta_{k_i}\,, \quad \zeta_{k_i}\coloneqq\frac{\ln (MN / \delta)}{{k_i}}\,.
\end{align*}    
\end{lemma}

The following lemma is the canonical elliptical potential lemma.
\begin{lemma}[Lemma 19.4,  \citet{lattimore2020bandit}]\label{lem:Elliptical potential lemma}
Let $M_{0}=\lambda_0 I \in \mathbb{R}^{d \times d}$ with $\lambda_0>0$ and $M_k=M_{k-1}+G_k$, where $G_k$ is positive definite with the maximum eigenvalue $\lambda_{\max}(G_k)\leq 1$ and $\operatorname{Tr}(G_k)\leq B^2$. Then
\begin{align*}
    \sum_{k=1}^K\operatorname{Tr}(G_kM_{k-1}^{-1})\leq2\ln\det(M_K)-2\ln\det(M_{0})\leq 2d\ln\left(1+\frac{KB^2}{d\lambda_0} \right)\,.
\end{align*}
\end{lemma}


The following lemma guarantees the concentration of the empirical feature covariance matrix and the version for fixed feature mapping $\phi(\cdot)$ is first proved by \citet{ZanetteCA21}. The proof of this lemma can be readily obtained by taking a union bound over any $\phi\in\Phi$ in the proof of Lemma 39 of \citet{ZanetteCA21}.
\begin{lemma}\label{lem: Concentration of the bonus term for roll-out policy}
Let $\lambda_{k_i}=\Theta(d \ln ({k_i}|\Phi| / \delta))=\Theta(d \ln ({k_i}M / \delta)), \forall i\in [N]$. Then simultaneously for all $i\in[N]$ and all $\phi\in\Phi$, with probability $1-\delta$, it holds that
\begin{align*}
\|\phi(s, a)\|_{\widehat{\Sigma}_{{k_i}, \phi}^{-1}}=\Theta\left(\|\phi(s, a)\|_{\Sigma_{\rho_{k_i} \times \bar{\pi}_{k_i}, \phi}^{-1}}\right)\,.
\end{align*}
\end{lemma}
The following is the canonical simulation lemma, which bounds the  difference between the performance of the same policy $\pi$ under two different environments and dates back at least to \citet{AbbeelN05}.
\begin{lemma}[Simulation lemma]\label{lem: simulation lemma}
Given two MDP models $\left(P^{\prime}, \ell-b\right)$ and $(P,\ell)$, for any policy $\pi$, it holds that
\begin{align}\label{eq:simulation_eq1}
    &V_{P^{\prime}, \ell-b}^\pi-V_{P, \ell}^\pi=\frac{1}{1-\gamma} \mathbb{E}_{(s, a) \sim d_{P^{\prime}}^\pi}\left[-b(s, a)+\gamma \left(P^{\prime}\left(\cdot\mid s, a\right)-P\left(\cdot\mid s, a\right)\right)^\top V_{P, \ell}^\pi\right]\,,
\end{align}
and
\begin{align}\label{eq:simulation_eq2}
    &V_{P^{\prime}, \ell-b}^\pi-V_{P, \ell}^\pi=\frac{1}{1-\gamma} \mathbb{E}_{(s, a) \sim d_{P}^\pi}\left[-b(s, a)+\gamma \left(P^{\prime}\left(\cdot\mid s, a\right)-P\left(\cdot\mid s, a\right)\right)^\top V_{P^\prime, \ell-b}^\pi\right]\,.
\end{align}
    
\end{lemma}

\end{document}